\documentclass[10pt,twocolumn,letterpaper]{article}

\usepackage{cvpr}
\usepackage[utf8]{inputenc} % allow utf-8 input
\usepackage[T1]{fontenc}    % use 8-bit T1 fonts
\usepackage{url}            % simple URL typesetting
\usepackage{booktabs}       % professional-quality tables
\usepackage{amsfonts}       % blackboard math symbols
\usepackage{nicefrac}       % compact symbols for 1/2, etc.
\usepackage{microtype}      % microtypography
\usepackage{makecell}% http://ctan.org/pkg/makecell
\usepackage{mathrsfs,amsmath}   %The amsmath package is included for \xrightarrow
\usepackage{enumitem}
\usepackage{multirow}
\usepackage{times}
\usepackage{epsfig}
\usepackage{graphicx}
\usepackage{amsmath}
\usepackage{amssymb}
\usepackage{rotating}
\usepackage{amsthm}
\usepackage{subfig}
\usepackage{color}
\usepackage[outercaption]{sidecap}
\usepackage{mathtools}
\usepackage[ruled,vlined]{algorithm2e}
\newtheorem{theorem}{Theorem}
\newtheorem{lemma}{Lemma}
\newtheorem{definition}{Definition}

\newtheorem{corollary}{Corollary}

\graphicspath{{./images/}{./tree/}}

\newcolumntype{P}[1]{>{\centering\arraybackslash}p{#1}}
\newcolumntype{M}[1]{>{\centering\arraybackslash}m{#1}}
\newcommand\numberthis{\addtocounter{equation}{1}\tag{\theequation}}

\def\ie{i.e. }

\newcommand{\tref}[1]{Table~\ref{#1}}
\newcommand{\Tref}[1]{Table~\ref{#1}}
\newcommand{\eref}[1]{Eq.~(\ref{#1})}
\newcommand{\Eref}[1]{Equation~(\ref{#1})}
\newcommand{\fref}[1]{Fig.~\ref{#1}}
\newcommand{\Fref}[1]{Figure~\ref{#1}}
\newcommand{\sref}[1]{Sec.~\ref{#1}}

\newcommand{\Ltwo}{\ell_2}

\renewcommand{\tabcolsep}{4pt}
\renewcommand{\arraystretch}{1.5}

\DeclareMathOperator*{\argminA}{arg\,min} % Jan Hlavacek
\newcommand{\var}{\mathrm{var}}

\newcommand{\E}{\mathbb{E}}
\newcommand{\bZ}{\mathbf{Z}}
\newcommand{\bX}{\mathbf{X}}
\newcommand{\bY}{\mathbf{Y}}
\newcommand{\bF}{\mathbf{F}}
\newcommand{\bM}{\mathbf{M}}

\newcommand{\bA}{\mathbf{A}}
\newcommand{\bB}{\mathbf{B}}
\newcommand{\ba}{\mathbf{a}}

\newcommand{\bc}{\mathbf{c}}

\newcommand{\btheta}{\boldsymbol{\theta}}
\newcommand{\bTheta}{\boldsymbol{\Theta}}

\newcommand{\bmu}{\boldsymbol{\mu}}

\newcommand{\bz}{\mathbf{z}}
\newcommand{\bx}{\mathbf{x}}
\newcommand{\by}{\mathbf{y}}
\newcommand{\bm}{\mathbf{m}}

\newcommand{\as}{\stackrel{a.s.}{=}}
\newcommand{\al}{\stackrel{a.s.}{<}}
\newcommand{\ag}{\stackrel{a.s.}{>}}

\newcommand{\dif}{d^2}

\usepackage[pagebackref=true,breaklinks=true,letterpaper=true,colorlinks,bookmarks=false]{hyperref}

\cvprfinalcopy % *** Uncomment this line for the final submission

\def\cvprPaperID{XXXX} % *** Enter the CVPR Paper ID here

\ifcvprfinal\fi
\hypersetup{draft} 
\begin{document}

%%%%%%%%% TITLE
\title{  Hierarchical Models:\\ Intrinsic Separability in High Dimensions}

\author{Wen-Yan Lin\\
{\tt\small daniellin@smu.edu.sg}
}

\maketitle

\begin{abstract}

 %high dimensional theories  beginning with the seminal works of Beyer
%~\cite{beyer1999nearest} and 
% Argawal~\cite{aggarwal2001surprising} who noticed the
% remarkable patterns formed by   high dimensional distributions. 
% Sine then, the patterns have been many re-interpretations of these patterns. Sometimes as a ``curse''~\cite{aggarwal2001surprising},

It has long been noticed that high dimension data exhibits strange patterns. This has  been variously interpreted as either a ``blessing''
or a   ``curse'', 
causing uncomfortable inconsistencies in the literature. We propose that these patterns
 arise from  an intrinsically hierarchical generative process. 
Modeling the process creates a web of constraints that 
reconcile many  different theories and results.
The model also implies   high dimensional data posses an innate separability
that can be exploited for machine learning. 
We demonstrate  how this permits the  open-set learning problem to be 
defined  mathematically, leading to  qualitative and quantitative improvements in performance. 
 \end{abstract}

\section{Introduction}

Despite the wide-spread use of  high dimensional data in computer vision and pattern recognition,  understanding of high dimensions has proven elusive.
In the dominant  mathematical paradigm,
 high dimensions are
 ``cursed''. The
 ``curse''  is rooted in a stochastic model where all
instances are generated from some distribution-of-everything.
This  causes almost all
instances to be equi-distant to any test point, making them
indistinguishable from each other and any machine learning
impossible~\cite{wiki,aggarwal2001surprising,domingos2012few}.

Yet, effective high dimensional machine learning does exist. 
Further, it utilizes the same equations  used to derive the  ``curse''.
However, 
by assuming multiple generative distributions~\cite{radovanovic2015reverse,lin2018dimensionality,tomasev2014role}, the  ``curse''  
turns into a ``blessing'' that  disentangles distribution instances. 
Both ``curse'' and ``blessing'' papers  appear mathematically sound but
 their conclusions are contradictory and   their assumptions are not  mutually exclusive. This creates a puzzle.

We believe  ``curse'' and ``blessing'' both stem from  patterns arising from an intrinsically hierarchical  data generation process. This would permit both their assumptions
to be simultaneously true and may also 
 explain why hierarchies  are 
a recurring theme in   language~\cite{simpson1961principles,stuessy2009plant},
 data-structures~\cite{windley1960trees}  and clustering algorithms~\cite{szekely2005hierarchical}. 
To understand this phenomenon, we develop a high dimensional \emph{hierarchical-model}.
The model  unifies  ``curse'' and ``blessing''
into  an elegant web of constraints that links all  generative distributions to each another and ultimately to a distribution-of-everything, forming  an over-arching framework that reconciles
many  theories and results.

%
%\begin{figure}
%	\centering
%	\includegraphics[width = 1.0\linewidth]{fig1.pdf}
%	\caption{Shell-learning 
% exploits the intrinsic structure of high dimensional data
% to learn, the    probability of an image belonging to a target class,
% from instances  of the target class.
%  Unlike classic 
%one-class svm~\cite{chen2001one} which  provides relative scores,
%shell-learning scores reflect   absolute  probability.
%Thus,   
%independently trained shell-learners can be coherently combined without re-training.
%The same scheme with  one-class svm classifies every image  as bird.     \label{fig:intro}}
%\end{figure}

\begin{figure}
	\centering
	\includegraphics[width = 1.0\linewidth]{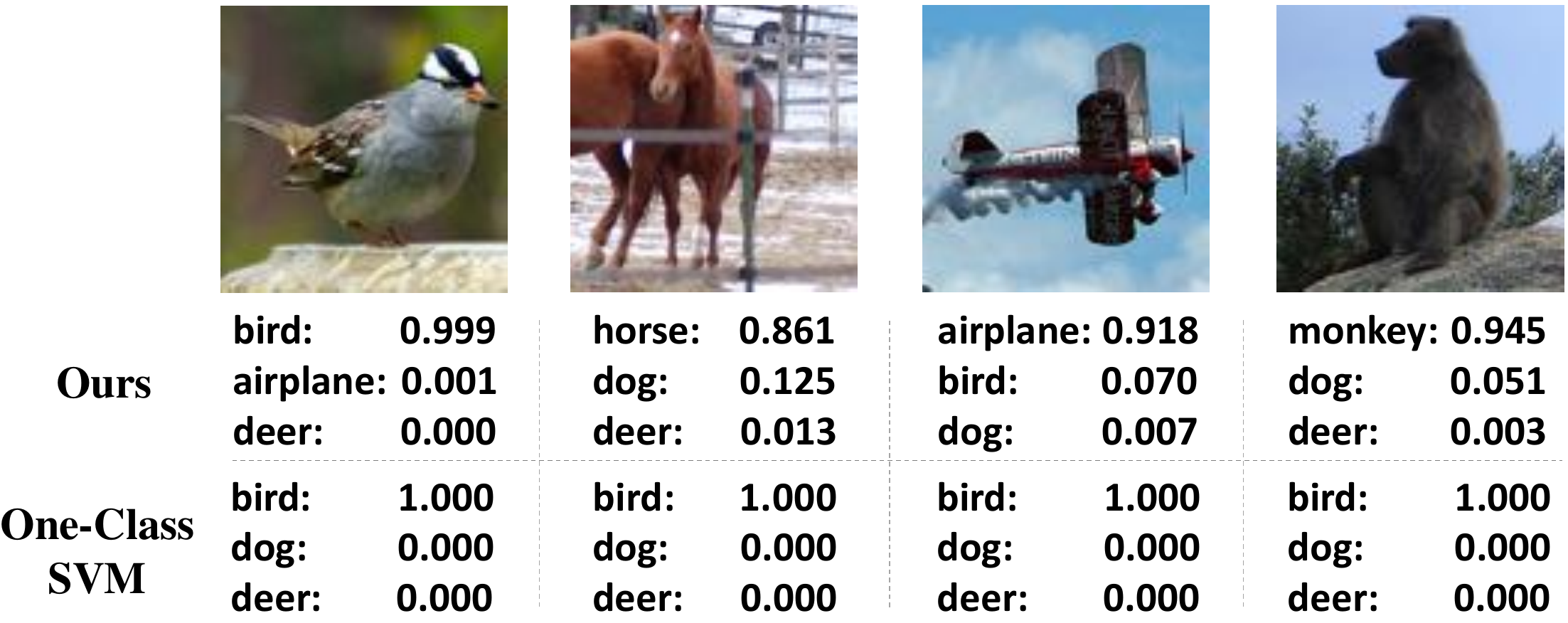}
	\caption{ 
Our formulation permits multiple one-class learners
to be fused into a multiclass classifier without re-training.
This is not possible with traditional one-class frameworks~\cite{kardan2016fitted}
as demonstrated with one-class SVM~\cite{chen2001one}, which
assigns all images to the bird class. 	
%One-class learning 	 estimates the 
%probability an image belongs to a target class,
% from instances  of the target class.
%	Classic one-class learners only  provide relative scores. Thus merging independently trained learners gives
%	ridiculous results. 
%We propose a theory of image generation in which
%one-class learning can be re-formulated as the discovery of   naturally distinctive shells. This provides absolute scores. 
		 \label{fig:intro}}
\end{figure}

Beyond explanatory power,   the \emph{hierarchical-model} also predicts almost-all instances of
a  distribution and its sub-distributions, lie on a distinctive-shell that other instances almost-never enter. 
Such distinctive-shells  can be estimated from  instances of the distribution. 
This relates current observations to all other potential observations, creating a  mathematical formulation for 
  open-set learning!

Open-set learning is one  of the most difficult   machine learning problems. In it,  a classifier trained on a small number of labeled classes must
deal effectively
with a much larger set of unlabeled classes~\cite{oza2019c2ae,yoshihashi2019classification,bendale2016towards}. 
Open-set learning is critical 
to handling  unforeseen circumstances   gracefully.
However, it is relatively unstudied because,  before the \emph{hierarchical-model},  
it appeared difficult (if not impossible) to  make  reasonable a priori assumptions on the properties of all possible unseen classes.  

We formulate the classic open-set problem of one-class learning in terms of the \emph{hierarchical-model}. This creates a shell-learner that is algorithmically similar to 
one-class SVM~\cite{chen2001one}. However, unlike traditional one-class learning, it 
 can approximate the absolute  probability of a class given an image.

%
%
%Thus, while  one-class svm only provides relative image rankings,
% shell-learners  can approximate the absolute  probability of a class given an image.

Such absolute scores are essential to operating in
unknown environments. 
It also means   scores of independently trained 
shell-learners can  be combined coherently  without  re-training.
This is especially important in 
life-long learning~\cite{kardan2016fitted} where constant  re-training can cause
potentially catastrophic forgetting and interference problems~\cite{parisi2019continual}.
An example 
 is illustrated in \fref{fig:intro}.

In summary, this paper:
\begin{itemize}[topsep=2pt]
\setlength{\parskip}{2pt}
\setlength{\itemsep}{0pt}
\item Proposes a \emph{hierarchical-model} that unifies ``curse'' and ``blessing'' approaches in  high dimensions;
\item Uses the  \emph{hierarchical-model} to predict the existence of distinctive-shells that make open-set problems amenable to mathematical analysis;
\item Demonstrates a one-class learner which estimates
the probability of a class given an image. Previously one-class learners only provided relative scores.  
\end{itemize}

\subsection{Related Works}

This paper lies at the intersection of many research fields. 
Its formulation follows  the long tradition of  statistical
machine learning. Examples include Expectation Miaxmization~\cite{dempster1977maximum,moon1996expectation,dellaert2002expectation}, 
 Latent Dirichlet Allocation (LDA)~\cite{blei2002latent,russell2006using, sivic2008unsupervised},  and Bayesian inference~\cite{box1965multiparameter,gelman2013bayesian,friedman1997bayesian}.  However, data density vanishes as dimensions increase, making probability density estimation ill-conditioned~\cite{scott2008curse}. Thus,  classic techniques often cannot scale to high dimensions.  The \emph{hierarchical-model}  avoids this problem by focusing on distance relations rather than  density.

While high dimensions are often  mathematically intractable,
computer vision has successfully utilized them in    deep-learning. Examples include  deep-learned image features~\cite{arandjelovic2016netvlad,he2016deep, simonyan2014very},
  Generative Adversarial Networks~\cite{dosovitskiy2016generating,goodfellow2014generative} and Variational Autoencoders~\cite{kingma2014auto}. Results are undoubtedly good but
   interpretation has proven  difficult. We  
  hope   \emph{hierarchical-models} make high dimensions
  and  (indirectly) deep-learners more  interpretable.  
  
Finally, the \emph{hierarchical-model} provides a statistical
 interpretation of high dimension Euclidean distances. As many
 data projection algorithms~\cite{pearson1901liii,maaten2008visualizing,mcinnes2018umap})
 seek to minimize or  preserve Euclidean distance, this has implications on 
 the interpretation of their results as discussed in \sref{sec:euc} of the supplementary.

%Our paper also  builds on a long series of  high dimensional theories  beginning with the seminal works of Beyer
%~\cite{beyer1999nearest} and 
% Argawal~\cite{aggarwal2001surprising} who noticed the
% remarkable patterns formed by   high dimensional distributions. 
% Sine then, the patterns have been many re-interpretations of these patterns. Sometimes as a ``curse''~\cite{aggarwal2001surprising},
% sometimes as a ``blessing''~\cite{lin2018dimensionality} and sometimes as
% just a useful constraint~\cite{radovanovic2015reverse,tomasev2014role}. The result is an  uncomfortably inconsistent literature. We hope  \emph{hierarchical-models} will be the first step towards  a unified framework.  
%
%
%Finally, our work is related to attempts to explain deep-learning's success,
%much of which is based on extremely   high dimensional features~\cite{arandjelovic2016netvlad,he2016deep, simonyan2014very}. 
%While the results are indisputably good, they are inconsistent    with the
% current  ``curse'' paradigm where  such  results should be impossible. 
% It is innately hard to explain why something works, with a theory which predicts it should not work. We hope
%that \emph{hierarchical-models} create a  mathematically model that is more consistent with empirical evidence, making  deep-learning  correspondingly easier to understand.

\section{Distributions in High Dimensions}
\label{sec:problem}

Our approach is based on the analysis of high dimensional distributions. Following the definition in~\cite{lin2018dimensionality}, a high dimensional distribution is
one whose random vectors are  \textit{quasi-ideal}, \ie there are a large number of dimensions, most of whom are independent. This is elaborated below,  with  notations adapted from  \cite{lin2018dimensionality}.

\small
%\sy{Need to standadize distance, distance = normalized sq dist? or Euclidean dist?}

\begin{definition}
	\mbox{ }\vskip -0.2cm
	\begin{itemize}
		\item $d^2(.)$ denotes an operator for normalized squared $\Ltwo$ norm, such that for $\mathbf{x}\in \mathbb{R}^k$, $d^2(\mathbf{x})=\frac{\|\mathbf{x}\|^2}{k}$. If $\bx$ is the difference between two vectors, we refer to $\dif(\bx)$ as \textbf{normalized squared difference} (NSD);
		
		\item $d(.)$ is the normalized $\Ltwo$ norm operator, $d(.) = \sqrt{d^2(.)}$; %=\frac{\|\mathbf{x}\|}{\sqrt{k}$;

		\item $\mathrm{S}(\bmu, r)$ denotes a thin shell centered at $\bmu$, with radius $r$.
	\end{itemize}
	
	\noindent Let $\bZ = \left [\bZ[1],\bZ[2], \hdots, \bZ[k] \right]^T$ denote a $k$ dimensional random vector where $\bZ[i]$ is a random variable,
	
	\begin{itemize}
		\item $\bZ$ is \textbf{high dimensional} if and only if  $\bZ$ is quasi-ideal~\cite{lin2018dimensionality}, i.e., as dimension $k\to\infty$, each dimension $\bZ[i]$ has finite fourth moment
		and a finite number of pairwise dependencies.
		
		\item  $d^2(.)$ operator  can  be applied to random vectors. $d^2(\mathbf{Z})=\frac{\|\mathbf{Z}\|^2}{k}$ is a random variable formed by averaging $\bz$'s
		squared elements;
		\item $\bmu_\bZ = \mathbb{E}(\mathbf{Z})= \left [ \E(\bZ[1]), \hdots, \E(\bZ[k]) \right]^T$ is a vector of each
		dimension's expectation;
		\item  $v_\bZ = \sum_{i=1}^k \frac{\var(\bZ[i])}{k}$ is the average variance;
		
		%\item If $\bZ$ is high dimensional, it means it is  quasi-ideal~\cite{lin2018dimensionality} (most dimensions
		%are independent), with dimensions $k \to \infty $;
		
		\item $\stackrel{a.s}{=}$ denotes almost-surely-equal. Thus,  the relation $P(d(\bX-\bc)=t) \to 1$ as $k\to\infty$, is written  $d^2(\bX-\bc) \as t$.
		
	\end{itemize}
	%
	%\item With slight abuse of notation, $\bX$ refers to both a random variable and the distribution it follows.
	\noindent  \textbf{Unit-vector-normalization.} As dimension $k \to \infty$, unit-vector-normalization causes
	individual entries to tend to $0$. For unit-vector-normalized data, the definitions of $\dif(.)$ and $v_\bZ$ are modified to avoid dividing by  $k$:
	\begin{itemize}
		\item $d^2(.)$ is a squared $\Ltwo$ norm, such that  $d^2(\mathbf{Z})=\|\mathbf{Z}\|^2$;
		\item $v_\bZ = \sum_{i=1}^k \var(\bZ[i])$, is the total variance.
	\end{itemize}
\end{definition}

\normalsize

Let $\bA$ and $\bB$ be two independent, high dimensional random vectors, with
respective mean and average variances $\bmu_\bA, \bmu_\bB$ and $v_\bA, v_\bB$.  Due to the law of large numbers, 
\textbf{the normalized squared difference between instances of $\bA$ and $\bB$  almost-surely depend, only on their mean and average variance}~\cite{lin2018dimensionality}\textbf{:}
\begin{equation}
\label{eq:main}
\dif(\bA- \bB) \as  v_\bA + v_\bB + \dif(\bmu_\bA - \bmu_\bB).
\end{equation}
This is key to understanding high dimensions.

Replacing $\bB$ in \eref{eq:main} by a distribution of  mean $\bc$ and zero variance,
 yields
\begin{equation}
\label{eq:main1}
\dif(\bA- \bc) \as  v_\bA + \dif(\bmu_\bA - \bc),  \quad \forall \bc \in \mathbb{R}^k.
\end{equation}

Let the \textbf{distribution-of-everything} represent a hypothetical, generative distribution  that fully explains the
creation of all natural images. 
If $\bA$ represents the  distribution-of-everything, 
\eref{eq:main1}
means  almost-all instances are 
equi-distant  to any point $\bc$. 
This has been used to argue that high-dimensions
are  ``cursed'' as it supposedly makes  instances
indistinguishable from each other~\cite{aggarwal2001surprising}. 
If true, the ``curse''  implies
no machine learning algorithm can work in high  dimensions.
However, if we consider a more complex  hierarchical generative model,
the conclusion changes. 

Our \emph{hierarchical-model} is based on three assumptions:
1) Images are instances of some generative distribution; 2) Generative distributions are high dimensional; 3) Except for the distribution-of-everything, each generative distribution is a sub-distribution of some parent.

The  ``curse'' does not directly apply to 
\emph{hierarchical-models} as it cannot account  for
pairwise distances between instances of the same sub-distribution.
This is because \eref{eq:main1}'s  almost-sure-equality   permits  exceptional events of infinitesimally small  probability.  
Thus if each sub-distribution $\bA_{\btheta}$
 accounts for only an infinitesimal  fraction of  $\bA$'s instances, 
the pairwise distance between its instances could be  exceptional.

While specific  sub-distributions may violate \eref{eq:main1}, 
in aggregate, they must conform to \eref{eq:main1}. 
This creates a web  of constraints   
linking all generative distributions (and images) to each other and  to the distribution-of-everything.

\section{The Hierarchical-Model}\label{sec:sub}

 We begin by defining the
relation between a distribution and its  sub-distributions. The definition
is then applied recursively to create the \emph{hierarchical-model}. 

Let $\bA$ be a high-dimensional random vector with probability density function  $f_{\bA}(\ba)$, which represents a distribution-of-everything. %, whose mean and average variance are $\bmu_\bX, v_\bX$, respectively.
$f_{\bA}(\ba)$ is a compound distribution, i.e., it is distributed according to some sub-distributions, with these distributions' parameters themselves
being random variables. $\bTheta=\{V_{\bTheta},  \mathbf{M}_{\bTheta}, \hdots, \mbox{other parameters} \}$ is a set of random vectors representing sub-distribution parameters.  $V_{\bTheta}$ and $\mathbf{M}_{\bTheta}$  represent the  average variance and mean parameters respectively. 
We write $\bA = \bA_{\bTheta}$ to denote that  parent distribution $\bA$ is explained by its sub-distribution parameters $\bTheta$ such that:
\begin{equation}
\label{eq:compound}
f_{\bA}(\ba) = f_{\bA_{\bTheta}}(\ba) = \int^{\infty}_{-\infty} f_{\bTheta}(\btheta) f_{\bA|\bTheta}(\ba|\btheta) d\btheta.
\end{equation}
where each sub-distribution is itself high-dimensional~\footnote{An example is compounding a Gaussian distribution with mean distributed according to another Gaussian distribution. Let each dimension of $\bA$ by independently  distributed with  $\bA[i]\sim N(\bM[i], \sigma^2_1)$, $\bM[i]\sim  N(\bmu[i], \sigma^2_2)$. The parameters of $\bA$ are 
	$\bmu_\bA=\bmu, v_\bA=\sigma^2_1+\sigma^2_2$.
}.

Let $\ba_{\btheta}$ be an instance of $\bA_{\bTheta}$, it is generated by a
two-step process:
1) Generate an instance of  sub-distribution parameters $\btheta$ from $\bTheta$. The random vector associated with $\btheta$ is  $\bA_{\btheta}$; 2) Generate   instance $\ba_{\btheta}$ from $\bA_{\btheta}$.

%Note that sub-distributions can be explained by their own sub-distributions.

We  recursively define  sub-distributions of sub-distributions through a hierarchical random process
$$\bA_{\bTheta^n} = ((((\bA_{\bTheta^{[0]}})_{\bTheta^{[1]}})_{\bTheta^{[2]}})_{\bTheta^{[3]}}\hdots)_{\bTheta^{[n]}} = \bA,$$ 
where operator $\bTheta^{[i]}$ represents the  random sampling of  a  sub-distribution, which in turn
forms the parent distribution for $\bTheta^{[i+1]}$. $\bA_{\bTheta^n}$ denotes an $n$ level hierarchical-process, with  
$\bA_{\bTheta^{[0]}}=\bA$.

%We can parameterize a selected sub-distribution in terms of
%its own sub-sub-distributions, one of which is selected and in turn parameterized by
%its own sub-sub-sub-distributions. The random vector thus defined
%is   $\bA_{\bTheta^n}=\bA$,
%where  $n$ denotes the number of generations.
If $\bA_{\bTheta^n}$ and $\bA_{\bTheta^m}$ denote independent runs of a hierarchical-process, starting from  $\bA$, from  \eref{eq:main},
\begin{equation}
\label{eq:vA}
\dif(\bA_{\bTheta^n} - \bA_{\bTheta^m}) \as \dif(\bA- \bA) \as 2v_{\bA}, \;\; m,n \in \mathbb{Z}^{+}_0.
\end{equation}

The instantiated sub-distribution parameters of a hierarchical-process are 
denoted as
$$\btheta^n = \btheta^{[0]}\btheta^{[1]}\btheta^{[2]}\hdots\btheta^{[n]},$$ 
where $\btheta^{[0]}$ is the parameter of $\bA$. $\btheta^n$ records the distribution parameters' values at every level, until the final distribution $\bA_{\btheta^{n}}$.  
%$\bA_{\btheta^n}$ represents the final random vector created by a hierarchical-process. $\bA_{\btheta^{[n]}}$
%is the random vector defined by the parameters of the last distribution in a hierarchical-process. 
Since the final random vector is completely characterized by the final distribution parameters, 
$\bA_{\btheta^n} = \bA_{\btheta^{[n]}}.$ 

A \emph{hierarchical-model} is a series of hierarchical processes, where at any $i^{th}$ stage, the $\bA_{\btheta^i}$ random vector  can be a  parent of multiple, new hierarchical processes. 

In the \emph{hierarchical-model}, any two data-points are instantiations of 
two independent hierarchical-processes starting from some,  most recent common ancestor. 
Let $\bA_{\btheta^n}$ denote their common ancestor and 
 $\bA_{\btheta^n\bTheta^i}$, $\bA_{\btheta^n\bTheta^j}$
 their distributions.
 % where $\bA_{\btheta^n}$ is the common ancestor.
From \eref{eq:vA},  the normalized squared difference (NSD) between any two instances is:
\begin{equation}
\label{eq:vTheta}
\dif(\bA_{\btheta^n\bTheta^i}  - \bA_{\btheta^n\bTheta^j} ) \as   2v_{\btheta^{[n]}}, \quad  \forall i,j, n \in \mathbb{Z}^{+}_0.
\end{equation}
This is summarized in Theorem \ref{theorem:L2}. 
\begin{theorem}
	\label{theorem:L2}
	(\textbf{Pairwise Difference}) The \emph{hierarchical-model}  predicts that 
	normalized difference between  any two data-points is
	almost-surely 	$\sqrt{2v}$, where $v$ is the average variance of their most recent common ancestor distribution.
\end{theorem}

\begin{figure}
	\centering
	\includegraphics[width = 1.05\linewidth]{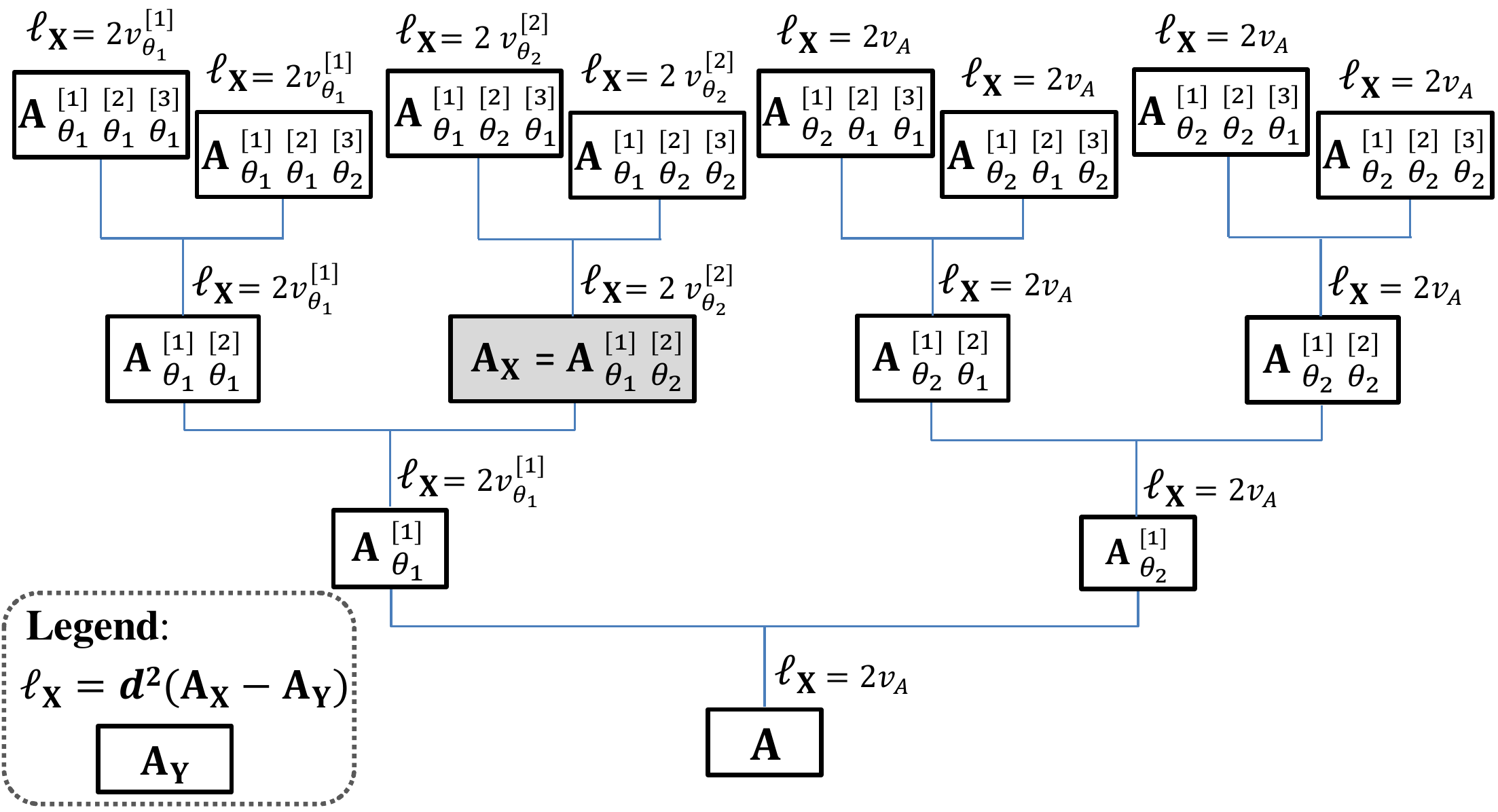}
	\caption{ A hierarchical-model with $\bA$ as   distribution-of-everything. 
		$\bA_\bY$  refers to the random vector enclosed by a rectangle.
		Normalized squared difference between $\bA_\bY$ and  the reference random vector, $\bA_\bX$, is denoted by
		$\ell_{\bX}$. From \eref{eq:vTheta},  $\ell_{\bX}$ is almost-surely 2 times the average variance of their most recent common ancestor. Each  path from root to leaf, represents a  hierarchical-process.}
	\label{fig:tree}
\end{figure}
%
%\begin{figure}
%	\centering
%	\includegraphics[width = 1.05\linewidth]{tree3b.pdf}
%	\caption{ A hierarchical-model with $\bA$ as the  distribution-of-everything. 
%		Pairwise normalized squared difference between instances of a distribution and $\bA_\bX$ is denoted by
%		$\ell_{\bX}$. From \eref{eq:vTheta},  $\ell_{\bX}$ is almost surely 2 times the average variance of their most recent common ancestor. Each  path from root to leaf, represents a  hierarchical-process.}
%	\label{fig:tree}
%\end{figure}

\Fref{fig:tree} illustrates a
 numerical example of the pairwise NSD.
Similar to the ``curse''~\cite{aggarwal2001surprising},
this formulation 
begins with a distribution-of-everything. However,
the NSD is not a constant. 
Instead, there is a defined  pattern of distances which algorithms
can exploit as a ``blessing''.
The distance patterns  also ensure Euclidean distance remains meaningful in high dimensions, as elaborated in the supplementary.  

In practice,     NSD is very often equal to ${2v_{\bA}}$ as  the  most recent common ancestor is often the distribution-of-everything. This creates a
paradox in which algorithms working directly with the affinity matrix
experience a ``curse''   but 
those dealing with nearest-neighbor distance do not. %\sy{do we need to mention spectral mtd?}     

\subsection{Parameter Constraints}\label{subsec:parameter_constraint}

This subsection  relates  distribution parameters across different hierarchies of the \textit{hierarchical-model}, creating the key constraints exploited in later sections. 

Thus far, a sub-distribution with identical parameters  to its parent is considered valid. Without loss of generality, we define a non-trivial hierarchical-process as one where sub-distributions  almost-surely have, finite difference in either mean or average variance, from its immediate parent: %i.e., 
\begin{equation}\label{eq:trival}
\begin{split}
&p(V_{\bTheta}^{[i+1]}   = V_{\bTheta}^{[i]}) =0  \\
\mbox{or} \quad & p(\dif(\mathbf{M}^{[i+1]}_{\bTheta} - \mathbf{M}^{[i]}_{\bTheta}) = 0)=0, \quad \forall i \in \mathbb{Z}^+_0.
\end{split}
\end{equation}
A non-trivial
\emph{hierarchical-model} is one where all  hierarchical-processes are non-trivial. 

For sub-distribution $\bA_{\btheta^{n}}=\bA_{\btheta^{[n]}}$, its mean and average variance
are denoted as $\bmu_{\btheta^{[n]}}, v_{\btheta^{[n]}}$ respectively. 
The mean and average variance of its hierarchical process descendants 
are denoted by random variables $\bM_{\btheta^n\bTheta^j},V_{\btheta^n\bTheta^j}$.

\begin{theorem}
	\label{theorem:mvde}
	(\textbf{Mean-Variance Constraint}) The mean  and average variance 
	of distributions generated by a hierarchical process with  $\bA_{\btheta^n}$ 
	as parent must adhere to the constraint:
	\begin{equation}
	V_{\btheta^n\bTheta^j} + \dif(\bM_{\btheta^n\bTheta^j}-\bmu_{\btheta^{[n]}}) \as v_{\btheta^{[n]}}, \quad \forall n,j\in \mathbb{Z}^+_0.
	\end{equation}
\end{theorem}

\begin{proof}
	Recall that $\bA_{\btheta^n\bTheta^j} = \bA_{\btheta^n}$. Thus, by requiring  the sub-distributions to, in aggregate, conform to the almost-surely constraint of their parent,  using \eref{eq:main1} yields: 
	\begin{equation}%\hskip -0.5cm
	\begin{split}\hskip -0.1cm
	\dif(\bA_{\btheta^n\bTheta^j}-  \bmu_{\btheta^{[n]}}) &\as \dif(\bA_{\btheta^n}- \bmu_{\btheta^{[n]}}) \as  v_{\btheta^{[n]}},\\
	&\as  V_{\btheta^n\bTheta^j} + \dif(\bM_{\btheta^n\bTheta^j}-\bmu_{\btheta^{[n]}})
	\end{split}
	\end{equation}
	\vskip -0.15cm
\end{proof} 

\begin{corollary}\label{col:v}
    %(Variance Relationship) 
	Distributions derived from a non-trivial hierarchical-process starting at $\bA_{\btheta^n}$,
	almost-surely have average variances
	less than  $v_{\btheta^{[n]}}$. \ie,
	\begin{equation}
	V_{\btheta^n\bTheta^j} \al v_{\btheta^{[n]}}, \quad \forall j,n \in \mathbb{Z}^+_0.
	\end{equation}
	That is, if $\btheta^{[i]}$ and  $\btheta^{[j]}$
	are distribution parameters from the same non-trivial hierarchical-process,
	almost-surely, $$v_{\btheta^{[i]}} >v_{\btheta^{[j]}}, \quad \forall i<j, \;i,j \in \mathbb{Z}^+_0.$$ 
\end{corollary}
\begin{proof}
	This follows from Theorem \ref{theorem:mvde} and \eref{eq:trival}.
\end{proof}
%\begin{proof}
%	Theorem \ref{theorem:mvde} implies
%	$V_{\btheta^n\bTheta^j} \ale v_{\btheta^{[n]}}$.  
%	As \eref{eq:trival} almost surely prohibits descendants of identical mean or average variance, the constraint tightens to $V_{\btheta^n\bTheta^j} \al v_{\btheta^{[n]}}$.  
%\end{proof}

Manipulating \eref{eq:main1}, Theorem \ref{theorem:mvde} and Corollary \ref{col:v}  leads to distinctive-shells, the paper's  main result.

\subsection{Distinctive-Shells}
\label{sec:one}

Consider a non-trivial hierarchical-process that starts
from the distribution-of-everything,  $\bA$ and terminates at 
 $\bA_{\btheta^n}$. Instances of  $\bA_{\btheta^n}$ and
 all its sub-distributions,  
 $\bA_{\btheta^n\bTheta^j}$ are  labeled $\alpha$.

From \eref{eq:main1},  $\forall \bc \in \mathbb{R}^k$
\begin{equation} \label{eq:bb}
\dif(\bA_{\btheta^n\bTheta^j} - \bc) \as  \dif(\bA_{\btheta^n} - \bc) \as v_{\btheta^{[n]}}  + \dif (\bmu_{\btheta^{[n]}}-\bc).
\end{equation}
Thus instances of $\alpha$, are almost-always  members of the set of shells,   $\{\mathcal{S}(\bc, \sqrt{ v_{\btheta^{[n]}} + \dif(\bmu_{\btheta^{[n]}}-\bc)}) \,|\, \bc \in \mathbb{R}^k\}$. The  minimum radius shell,  $\mathcal{S}(\bmu_{\btheta^{[n]}}, \sqrt{v_{\btheta^{[n]}}})$, is termed the \textbf{distinctive-shell} and is of unique significance. 

We represent the ancestors of  $\bA_{\btheta^n}$
as $\bA_{\btheta^m}, m < n$. 
All possible non-$\alpha$ instances can  be represented  by $\bA_{\btheta^m\bTheta^i}$.  \Eref{eq:main1}
 relates all such instances to $\bmu_{\btheta^{[n]}}$
through
\begin{equation}
\label{eq:hier_dist}
\begin{split}
\dif( \bA_{\btheta^m\bTheta^i} - \bmu_{\btheta^{n}}) &
\as \dif(\bA_{\btheta^m} - \bmu_{\btheta^{n}})\\
&\as v_{\btheta^{[m]}} +  
\dif(\bmu_{\btheta^{[m]}}-\bmu_{\btheta^{[n]}}).
\end{split}
\end{equation}

From Theorem \ref{theorem:mvde}, almost-surely,
\begin{equation}
\label{eq:hier_para}
v_{\btheta^{[n]}} +d^2(\bmu_{\btheta^{[m]}}-\bmu_{\btheta^{[n]}}) =v_{\btheta^{[m]}}, \quad \forall m<n,\; m,n\in \mathbb{Z}^+_0.
\end{equation}

Combining \eref{eq:hier_para} and \eref{eq:hier_dist}, almost-surely,
\begin{equation}
\label{eq:seperate}
\dif( \bA_{\btheta^m\bTheta^i} - \bmu_{\btheta^{[n]}}) \as 2v_{\btheta^{[m]}} -v_{\btheta^{[n]}}, \; \forall m<n,\; m,n\in \mathbb{Z}^+_0.
\end{equation}

Thus for each $\bA_{\btheta^m}$ ancestor of $\bA_{\btheta^n}$,  Corollary \ref{col:v}  almost-surely implies
\begin{equation}
\label{eq:shell}
\begin{split}
& \dif(\bA_{\btheta^{0}\bTheta^i} - \bmu_{\btheta^{[n]}})\ag \dif(\bA_{\btheta^{1}\bTheta^i} - \bmu_{\btheta^{[n]}}) \ag\hdots \\
\ag & \dif(\bA_{\btheta^{n}\bTheta^i} - \bmu_{\btheta^{[n]}})\as v_{\btheta^{[n]}}. 
\end{split}
\end{equation}

Recall that $\bA_{\btheta^m\bTheta^i}$ represents all
possible non-$\alpha$ instances. Thus \eref{eq:shell}
guarantees  non-$\alpha$ instances almost-surely fall outside  the distinctive-shell. 
Further, \eref{eq:bb} enables us to estimate  
the distinctive-shell by fitting the tightest possible shell to instances of  $\alpha$.
This makes it  possible to algorithmically relate
current observations with all other potential observations!

%
%Finally, the distinctive-shells predicts that each sub-distribution
%should be a member of its parent's distinctive-shell.  Thus, instances of sheep-dog should be on the distinctive-shells of dogs, which should be on the distinctive-shell of animals. This forms  a testable prediction we perform in []. 

\section{Shell Based Learning}

This section applies the theory of  distinctive-shells
 to  practical open-set problems, where a concept detector  trained   
on limited  data   must generalize to unknown environments. 

%
%The previous sections focused on the \emph{hierarchical-model} theory. 
%This  section  applies the  \emph{hierarchical-models} to an open-set problem. Our problem
%of choice is one-class learning, where a concept detector is trained   
%from sample instances of the  concept. Sub-sections are as follows: \Sref{sec:bayes} covers formulation;
%\sref{sec:re-norm}  coordinate frame choice;
%\sref{sec:feat}  feature selection; \sref{sec:stacked}  implementation and \sref{sec:err}
%sources of error.
%
%
%Instead of solving the problem directly, we pose it as a quasi-open-set problem where we constrain the problem in realistic and easily achievable ways such that the final goal. Thus paper focuses on creating an image based class detector using provided training data  and whatever information can be easily crawled from the net. 

\subsection{Formulation}
\label{sec:bayes}

\noindent{\textbf{Unit-Vector-Normalization}}

Despite the complexity of the image formation process,  generative distributions of  digital images
cannot be modeled as high dimensional. This is because the brightness of all pixels in an image is correlated by a  common exposure setting, making  dimensions dependent. 
Exposure induced scaling  can be removed by unit-vector-normalization
that divides each data point by its $\ell_2$ norm, converting all data to unit vectors.

Suppose $\bA$ is a \emph{hierarchical-model's} distribution-of-everything. If each
$\ba_i$ instance,  is perturbed by a random  scalar such that 
$\widetilde{\ba_i} = s_i\ba_i$,
after unit-vector-normalization, 
$\frac{\widetilde{\ba_i}}{\|\widetilde{\ba_i}\|} = \frac{\ba_i}{\|\ba_i\|}$,
which is identical to unit vector normalization on the un-perturbed instances.
Thus,  the effect of normalization of perturbed instances can be understood
by analyzing the normalization of un-perturbed instances.

Setting $\bc$ to a zero vector, $\mathbf{0}$, in \eref{eq:main1} yields
\begin{equation}
\label{eq:scale}
\dif(\bA_{\bTheta^n}) \as \dif(\bA) \as v_\bA +\dif(\bmu_\bA)= \lambda_\bA,
\end{equation}
where $\lambda_\bA$ is a constant. Thus, unit-vector-normalization of \emph{hierarchical-model} data, is almost-surely equivalent to  scaling  the entire  \emph{hierarchical-model}  by a constant factor.

From the definition of $\dif(.)$ and \eref{eq:scale}, 
\begin{equation}
\dif(\bA) = \frac{\|\bA\|^2}{k} \as \lambda_\bA \Rightarrow \|\bA\| \as \sqrt{\lambda_\bA k},
\end{equation}
where $k$ is the number of dimensions.
If $\widehat{\bA}$ be the normalized version of $\bA$, after normalization, $\|\widehat{\bA}\|^2 \as \frac{\|\bA\|^2}{\lambda_\bA k}$. 
Since $k$ is already part of this normalization, 
for post unit-vector-normalized data, we 
 modify the definition of $\dif(.)$ and $v$ operators,  detailed in \sref{sec:problem}, to avoid another  division by $k$. This extends  all previous high dimensional analysis to post unit-vector-normalized data. 
\newline

\noindent{\textbf{Statistical Framework}}

Let $\bx \in \mathbb{R}^k$ be a  data point and  $y \in \mathbb{Z}^+_0$  some label. The goal of  learning is to  estimate $p(y|\bx)$ for all  $y$ and $\bx$. 

If  
  instances of a $\bA_{\btheta^n\bTheta^i}$ correspond to  label $\alpha$,
from \eref{eq:shell}, $\alpha$ instances almost-surely 
  lie on  distinctive-shell $\mathcal{S}(\bmu_{\btheta^{[n]}},\sqrt{v_{\btheta^{[n]}}})$,  which exclude almost-all other instances. Thus, if  $\bx$ is modified to 
  be the distance from the distinctive-shell, we have the following:
\begin{equation}
\label{eq:human}
\begin{split}
p(y=\alpha| \mathbf{x}=0) = 1, \quad p(\mathbf{x}=0| y=\alpha) = 1 \\
p(y\neq \alpha| \mathbf{x}=0) = 0, \quad p(\mathbf{x}=0| y\neq \alpha) = 0 \\
p(y=\alpha| \mathbf{x} \neq 0) = 0, \quad p(\mathbf{x} \neq 0| y=\alpha) = 0 \\
p(y\neq\alpha| \mathbf{x} \neq 0) = 1, \quad p(\mathbf{x} \neq 0| y\neq\alpha) = 1 \\
\end{split}
\end{equation}
This can be summarized as $p(y =\alpha |\bx) =p(\bx|y=\alpha)$. In practice, we use:   
\begin{equation}
\label{eq:shell_s}
p(y =\alpha |\bx) \approx p(\bx|y=\alpha).
\end{equation}

While \eref{eq:shell_s} does not represent all perturbations, it is still a remarkable approximation as it
frees us from assuming a prior (that most statistical formulations require). This is key to permitting  
open-set formulations for unknown  operating environments. 
\newline

\noindent\textbf{Shell Fitting}

Given  $l$ instances of class $\alpha$, 
$\{\mathbf{f}_0, \mathbf{f}_1, \hdots, \mathbf{f}_l\}$,
from \eref{eq:shell},  
the distinctive-shell $\mathcal{S}(\bmu_{\btheta^{[n]}},\sqrt{v_{\btheta^{[n]}}})$
is estimated by minimizing 
\begin{equation}
\label{eq:shell_min}
\argminA_{\{\bmu, v\}}\frac{1}{l} \textstyle\sum_{i=1}^l \|\|\mathbf{f}_i-\bmu\|^2 -v\|^2 + \lambda \|v\|^2.
\end{equation}
$\bmu,v$ are the respective estimates of $\bmu_{\btheta^n},v_{\btheta^n}$  and  $\lambda$  is a regularizer that encourages small shells.
Let $\bx_i = \|\mathbf{f}_i-\bmu\|^2$. 
$p(\bx|\alpha)$ is estimated   by applying a parzen window to $\bx_i$ instances. This approximates  $p(\alpha|\bx)$ in \eref{eq:shell_s}.

Note: One-class SVM libraries~\cite{chen2001one} can also estimate  $\bmu$ but  
minimizing \eref{eq:shell_min} is better in extreme cases, as shown in supplementary \sref{sec:mean}.

\subsection{Re-normalization}
\label{sec:re-norm}

This section shows the impact of coordinate frame choice on learning.
We assume  data  is unit-vector-normalized.

Similar to \sref{sec:one},  let $\bA_{\btheta^m}$ be
some ancestor distribution of $\bA_{\btheta^n}$. 
From  
\eref{eq:seperate}, the gap in distinctive-shell distance between
instances of  $\bA_{\btheta^m\bTheta^i}$   and $\alpha$ instances of  $\bA_{\btheta^n\bTheta^j}$ is:
\begin{align*}
%\begin{split}
G_{mn} &= \dif( \bA_{\btheta^{m}\bTheta^i} - \bmu_{\btheta^{[n]}}) - \dif( \bA_{\btheta^{n}\bTheta^j} - \bmu_{\btheta^{[n]}}) \\
&\as 2(v_{\btheta^{[m]}}- v_{\btheta^{[n]}}), \quad\; m \leq n. \numberthis
%\end{split}
\end{align*}
For convenience, we shorten the expression to
\begin{equation}
g_{mn} = 2(v_{\btheta^{[m]}}- v_{\btheta^{[n]}}), \quad m \leq n.
\end{equation}

In theory, a finite  gap suffices for  separability. However,  due to  noise,  ensuring  the largest possible gap is important.  
Gaps can be altered through  \textbf{re-normalization}, where each dimension is translated by some value and 
 resultant vectors unit-normalized again to  magnitude 1.

Given  the mean of the distribution-of-everything, $\bmu_\bA$,   we can re-normalize  by subtracting $\bmu_\bA$ from all instances and dividing  by their new magnitude.   From \eref{eq:main1}, $\dif(\bA_{\bTheta^k}-\bmu_\bA ) \as v_\bA$. Thus this step is equivalent to re-centering  almost-all   instances 
and dividing them by a scalar $s_\bA = \sqrt{v_\bA}$.  Let   
$\widetilde{\bA}_{\widetilde{\btheta^n}\bTheta^i}, \widetilde{\bA}_{\widetilde{\btheta^m}\bTheta^j}$ be the re-normalized sub-distributions. The gap becomes:  
\begin{equation}
\label{eq:global_norm}
\widetilde{g_{mn}} =  \frac{2(v_{\btheta^{[m]}}- v_{\btheta^{[n]}})}{s_\bA^2} =\frac{2(v_{\btheta^{[m]}}- v_{\btheta^{[n]}})}{v_\bA}, \;\; m  \leq n.
\end{equation}

As data has been unit-vector-normalized,  $\dif(\bA-\mathbf{0}) \as v_\bA+\dif(\bmu_\bA-\mathbf{0})=1$ and  
 $v_\bA\leq 1$. Thus re-normalizing with $\bmu_\bA$ is guaranteed to not reduce  the gap. If $\bmu_\bA$ is far from zero,  $v_\bA$ can be very small,
 causing  a  corresponding huge improvement  in the gap
for almost-all pairs of distributions $m,n$.  
  Re-normalization with
 $\bmu_\bA$ is identical to standard normalization procedures in machine learning
 and helps explain the importance of this  step in general machine-learning.

The gap can be further manipulated if we have knowledge of the   ancestors of  $\bA_{\btheta^{n}}$.
This  is summarized as:  
\begin{corollary}
\label{col:re-norm}
(Re-Normalization) Let   $\bA_{\btheta^l}$ be some ancestor of $\bA_{\btheta^{n}}$, \ie $ l \in \mathbb{Z}^+_0, l \leq n$. If we re-normalize data with $\bmu_{\btheta^{l}}$, almost-surely,
\begin{equation}
\label{eq:re-norm}
    \widetilde{g_{mn}} = 
\begin{dcases}
\frac{2(v_{\btheta^{[m]}}- v_{\btheta^{[n]}})}{v_{\btheta^{[l]}}},& \text{if } \; 0\leq l \leq m \leq n,\\
\frac{2(v_{\btheta^{[l]}}-v_{\btheta^{[n]}})}{v_{\btheta^{[l]}}}, & \text{if } \; 0\leq m \leq l \leq n .
\end{dcases}
\end{equation}
\end{corollary}

\begin{proof}\renewcommand{\qedsymbol}{}
	 Proof is in \sref{sec:re-norm_proof} of the supplementary.
\end{proof}

  From Corollary \ref{col:v}, 
$$1 \geq  v_\bA >  v_ {\btheta^{1}} > v_ {\btheta^{2}} > v_ {\btheta^{3}} \hdots  > v_ {\btheta^{n}}.$$ Thus, the first case of \eref{eq:re-norm} implies for instances of 
 $\bA_{\btheta^{m}\bTheta^j}$, where $m$ is between $l$ and $n$,
  re-normalizing  with   $\bmu_{\btheta^{[l]}}$,
 increases separability from instances of $A_{\btheta^n\bTheta^j}$ which correspond to class $\alpha$. 
 The second case of  \eref{eq:re-norm}  shows this comes at the expense of  separability  when
 $m$ is less than $l$.  This is equivalent to stretching contrast (separation gap) of $\bA_{\btheta^{l}}$  descendants  at the expense of compressing contrast of its ancestors. 
 In the extreme case, where $l=n$, the gap between $\bA_{\btheta^{n}}$ and all its ancestors vanishes. Thus,
 standard normalization with one-class learning causes
  disastrously bad results, as shown in supplementary \sref{sec:norm_effect}.

These trade-offs motivate \sref{sec:stacked}'s training  of a stack of  one-class learners 
by re-normalizing data  with different ancestor means.  This
enables the exceptionally fine retrieval demonstrated in \fref{fig:ss_human}.

\subsection{Magic of Deep Learned Features}
\label{sec:feat}

Using deep-learned features is known to ``magically'' improves learning results.   This phenomenon can be studied in the context of \emph{hierarchical-models}. 

Part of the improvement may arise from an  innately better
representation (that we cannot explain), which reduces the ratio of  within-class variance 
to between-class variance, essentially amplifying the gap in \eref{eq:re-norm}.

However, \sref{sec:re-norm} shows coordinate frame choice  is also important
to learning. In particular, it is ideal if the 
mean of the distribution-of-everything is at $\mathbf{0}$. 
Deep-learned do indeed seem to have made this choice 
and changing their coordinate frame significantly reduces performance, as shown in \sref{sec:magic_sup} 
of the supplementary. We also 
 leverage this property in our one-class learning algorithm in \sref{sec:stacked}.

\subsection{Implementation}
\label{sec:stacked}
We term our overall algorithm \textbf{Shell-Stacked (SS)}. Details  are as follows.

\textit{Training (Algorithm \ref{algo:train}):}
Given  training images from a target class $\alpha$, we
store  their unit-vector-normalized  ResNet features~\cite{he2016deep} as $\bF$. Optionally, we are given a set of ancestor distribution means or ``external knowledge'', denoted as $\bM = [\bm_1,\bm_2, \hdots, \bm_K]$.  If $\bM$ is not given, we estimate it by crawling $K-1$ semantic concepts related to the target class $\alpha$ from the  Internet. We first estimate the mean,
$\bm$ of  training data  $\bF$. For crawled concepts, we compute the mean of their image
features and rank them from nearest to furthest from $\bm$, creating
$ [\widetilde{\bm_1},\widetilde{\bm_2}, \hdots, \widetilde{\bm_{K-1}}]$.
By taking a weighted average of $\bm$ with $\widetilde{\bm_i}$, 
we estimate  $\bm_i$ , an approximation of a parent distribution mean.  This is given by  $\bm_i = \frac{\bm +  \sum_{j=1}^i \widetilde{\bm_j}}{i+1}, \; i\geq 1, \bm_K =\mathbf{0}$.
The final zero vector 
corresponds to assuming ResNet feature's 
distribution-of-everything has  mean, $\mathbf{0}$, as described in \sref{sec:feat}.
%\noindent\textbf{Algorithm for Stacked One-class Learning}

If internet crawling is unavailable,  set $\bM = [\mathbf{0}]$. This is termed \textbf{Shell-One (SO)} and is a special case of SS.

We re-normalize all input data with 
 each  $\bm_i$ and  fit a tight shell to the re-normalized data, using 
\eref{eq:shell_min}, to estimate shell
centers, $\bmu_i$ and  probability density function
$p_i(\bx_i = x_i |\alpha)$. The output is a list  parameters $\bm_i$, $\bmu_i$ and $p_i$.

\textit{Testing (Algorithm \ref{algo:test}):} $\by$ is a unit-normalized ResNet feature. Compute $\widetilde{\by}_i$ by re-normalizing $\by$ with  $\bm_i$. The score of $\by$ is  $\frac{1}{K}\sum_{i=1}^K p_i(\dif(\bmu_i-\widetilde{\by}_i)|\alpha)$. 

\subsection{Errors in Shell Learning}
\label{sec:err}

In theory, shell fitting results should be perfect.  In  practice, 
there are three major error sources:

Firstly, computationally tractable features do not have infinite dimensions, 
with even    deep-learned features  having only a few hundred linearly independent dimensions. Thus data typically form  cloudy 
approximates of the ideal shell;  

Secondly, there can be semantic gaps between  labels and  distributions. An extreme example would be the label apple mapping to both the iPhone  and the fruit.   This causes an excessively large  common shell that
encompasses many unrelated instances. The problem can be alleviated by clustering. However, that is beyond this paper's scope;

Thirdly, while  \eref{eq:shell_min}'s regularization provides robustness its estimated shells are too small, causing some true instances to fall outside the shell.

\begin{algorithm}[h!]
\small
\hskip -0.4cm
\KwIn{1) Training features $\mathbf{F} = [\mathbf{f}_1,\mathbf{f}_2, \hdots, \mathbf{f}_N]$; \\2) (optional) "External knowledge" of ancestor distributions mean $\mathbf{M} = [\mathbf{m}_1,\mathbf{m}_2, \hdots, \mathbf{m}_K]$  }
\nl $\mathbf{S}$ = \textbf{empty list}\;
%\nl \uIf{\textbf{not} K == 0}{
\nl     \For{$i = 1$ to $K$}{
\nl  \hskip -0.15cm $\widetilde{\mathbf{F}_i}$ = \textbf{UNIT-VECTOR-NORM}($\mathbf{F} - \mathbf{m}_i$)\footnotesize \,\# re-normalize\;\small
\nl  \hskip -0.15cm  $\mathbf{\bmu}_i$ = \textbf{SHELL-FIT}($\widetilde{\mathbf{F}}$) \,\footnotesize\# estimate shell-center with Eq.(\ref{eq:shell_min})\; \small
\nl  \hskip -0.15cm  $\bX_i = \{\bx_j | \bx_j = \|\widetilde{\mathbf{f}_{ij}} -\bmu_i\|^2\}$, \,where $\widetilde{\mathbf{f}_{ij}}=\widetilde{\mathbf{F}_i}[:,j]$\;
\footnotesize \hskip 2.0cm \# $\bx_j$ represents distance from shell-center\; \small
\nl \hskip -0.15cm        $p_i(\bx)$ = \textbf{DENSITY-ESTIMATION}($\bX_i$)\;
\nl  \hskip -0.15cm       $\mathbf{S}$.append($[\mathbf{m}_i, \bmu_i, p_i(\mathbf{x})]$)\;
}
\nl \Return $\mathbf{S}$
\caption{{\bf Shell-Training} \label{algo:train}}
\end{algorithm}

\begin{algorithm}[h]
	\small
	\hskip -0.4cm
	\KwIn{1) Test feature $\mathbf{f}$;   2)  Distinctive-shells of $\alpha$, $\mathbf{S} = [\mathbf{s}_1, \mathbf{s}_2, \hdots, \mathbf{s}_K].$}
	\nl y = 0\;
	\nl     \For{$i = 1$ to $K$}{
		\nl         $[\mathbf{m}_i, \bmu_i, p_i(\mathbf{x})]=\mathbf{s}_i$\;
		\nl         $\widetilde{\mathbf{f}_i}$ = \textbf{UNIT-VECTOR-NORM}($\mathbf{f} - \mathbf{m}_i$)\;
		\nl         $\mathbf{x}_i = \|(\widetilde{\mathbf{f}_i} -\bmu_i)\|^2$\;
		\nl         $y$ = $y + \frac{p_i(\mathbf{x}_i)}{K}$
	}
	\nl \Return $y$ \quad \quad \footnotesize \# an estimate of $p(\alpha|f)$. \small
	\caption{{\bf Shell-Testing} \label{algo:test}}
\end{algorithm}

\renewcommand{\arraystretch}{1.2}
\renewcommand{\tabcolsep}{2.2pt}
\begin{table}[htp]
	\footnotesize
	\centering
	\begin{tabular}{l|c|c|c|c|c}
		%\begin{tabularx}{\columnwidth}{lccccc}
		\toprule
		%		&\multicolumn{11}{c}{STL-10~\cite{coates2011analysis} (ResNet-50 Features)} \\
		&\multicolumn{5}{c}{\textit{Average AUROC for each  data-set}}\\
		\hline
		&\makecell{Fashion\\-MNIST}   &\makecell{STL-10}& \makecell{Internet\\ STL-10} & \makecell{MIT-\\Places} & ASSIRA \\
		\hline
		OC-SVM~\cite{chen2001one} & 0.892 (I) & 0.799 (R) & 0.557 (R)& 0.765 (R) & 0.824 (R)\\
		SO-Ours   & 0.911 (I) & 0.958 (R) & 0.948 (R) & 0.910 (R) & 0.964 (R)\\
		SS-Ours & \textbf{0.953} (I)& \textbf{0.987} (R) & \textbf{0.975} (R) & \textbf{0.983} (R) & \textbf{0.994} (R)\\
		OC-NN~\cite{chalapathy2018anomaly} &   0.851 (I) & 0.949 (R) & 0.932 (R) & 0.895 (R) & 0.932 (R)\\
		Deep A.Det.~\cite{golan2018deep} & 0.935 (I)& 0.730 (I) & 0.717 (I) & 0.722 (I) & 0.888 (I)\\
		DSEBM~\cite{zhai2016deep}              & 0.884 (I)& 0.571 (I) & 0.560 (I) & 0.613 (I) & 0.516 (I)\\
		DAGMM~\cite{zong2018deep}                & 0.518 (I)& 0.554 (I)& 0.517 (I) & 0.530 (I) & 0.485 (I)\\
		AD-GAN~\cite{deecke2018anomaly}               & 0.884 (I)& 0.602 (I) & 0.555 (I) & 0.499 (I) & 0.534 (I)\\
		\hline
	\end{tabular}
	\caption{AUROC score of various one-class learners. Inputs are indicated with (I) for raw images and (R) for  ResNet features.   Shells-Stacked (SS) results are  noticeably good, 
		indirectly validating the \emph{hierarchical-model} used in its design. \label{tab:all}}
\end{table}
\normalsize

\section{Experiments}

This section  focuses on checking     the
 \emph{hierarchical-model's}  predictions and   quantitative analysis of  shell based  learning.   
 
\renewcommand{\arraystretch}{1.0}
\begin{table*}[tp]
	\begin{minipage}[t]{.8\textwidth}
		\centering
		\small
		\begin{tabular}{cccccccccccc}
			\hline
			& \makecell{airplane\\(ship)} & \makecell{bird\\(cat)} & \makecell{car\\(truck)} & \makecell{cat\\(dog)} & \makecell{deer\\(horse)} & \makecell{dog\\(cat)} & \makecell{horse\\(dog)} & \makecell{monkey\\(dog)} & \makecell{ship\\(truck)} & \makecell{truck\\(car)} & Average \\
			\hline
%			SO-Ours & 0.884 & 0.958 & 0.887 & 0.795 & 0.879 & 0.887 & 0.911 & 0.928 & 0.932 & 0.850 & 0.891 \\
%			SS-Ours & 0.972 & 0.988 & 0.953 & 0.902 & 0.944 & 0.949 & 0.949 & 0.972 & 0.982 & 0.954 & 0.957 \\
%			%
			
			SO-Ours & 0.886 & 0.970 & 0.842 & 0.713 & 0.884 & 0.746 & 0.701 & 0.929 & 0.897 & 0.619 & 0.819\\
			SS-Ours& 0.972 & 0.994 & 0.933 & 0.826 & 0.954 & 0.927 & 0.914 & 0.969 & 0.942 & 0.784 & 0.922 \\
			\hline
		\end{tabular}
		
		\begin{tabular}{cccccccccccc}
			\hline
			& \makecell{abbey\\(alley)} & \makecell{airport terminal\\(amusement park)} & \makecell{alley\\(airport terminal)} & \makecell{amusement park\\(airport terminal)} & \makecell{aquarium\\(amusement park)} & Average \\
			\midrule
			SO-Ours & 0.829 & 0.909 & 0.895 & 0.703 & 0.824 & 0.832 \\
			SS-Ours & 0.965 & 0.959 & 0.981 & 0.938 & 0.954 & 0.959 \\
			\hline
		\end{tabular}
	\end{minipage}\hfill
	\begin{minipage}[t]{.2\textwidth}\vspace*{-18pt}%
		\setlength{\abovecaptionskip}{0pt}%
		\caption{AUROC on the most difficult class pairs of Internet STL-10 and MIT-Places. 
Shells-Stacked's (SS) gain over Shells-One (SO) is much more noticeable.\label{tab:hard}}%
	\end{minipage}
\end{table*}

%\begin{figure}[htp]	
%\centering
%	\begin{minipage}[t]{ 0.7\linewidth}
%		\centering
%		\small
%		\includegraphics[width = 0.7\linewidth]{pr_curve_SO-SS-OCSVM-new.pdf}
%			\end{minipage}
%				\begin{minipage}[t]{.2\linewidth}\vspace*{-100pt}%
%		\setlength{\abovecaptionskip}{0pt}%
%		\caption{AUROC on the most difficult class pairs. 
%Shells-Stacked's (SS) gain over Shells-One (SO) is markedly greater, reflecting   finer recognition ability.\label{tab:hard}}%
%	\end{minipage}\hfill
%
%\end{figure}

\begin{figure}[htp]	
	\centering
	\small
	\includegraphics[width = 0.8\linewidth]{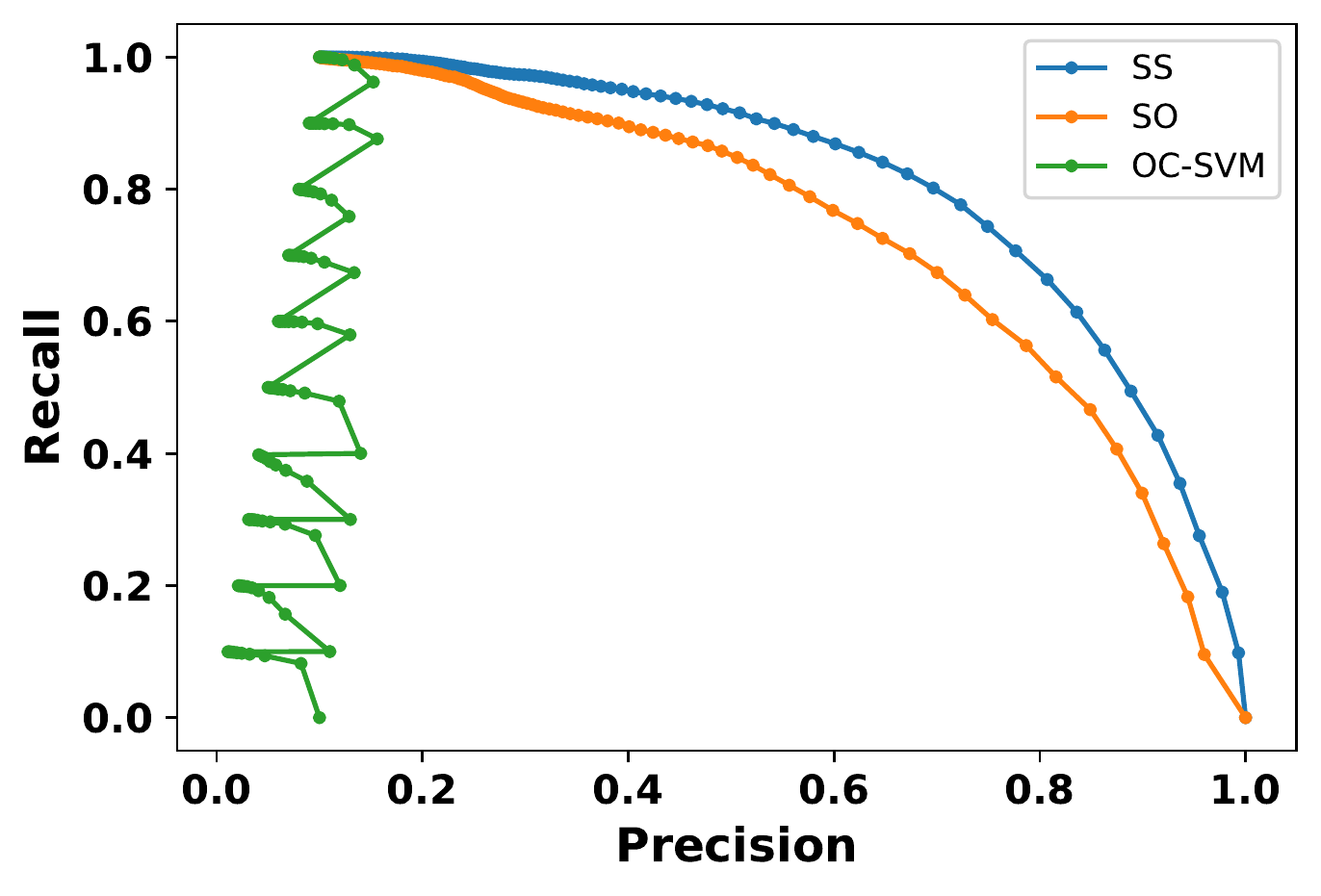}
%	\begin{tabular}{cccc}
%		\small
%		\centering
%		&SS & linear-SVM &  SVC\\
%		\hline
%		Average Precision &0.862 & 0.899 & 0.437\\
%		\hline
%	\end{tabular}
	\caption{ Recall vs Precision from  
	integrating multiple one-class learners for 
	multiclass
	classification. This can be achieved with 
	 shell-learners  like SO and SS. With traditional one-class learners like OC-SVM,  each learner has a unique range of values, creating the vertical zig-zag line seen above. 
		  \label{fig:global}}
\end{figure}

\subsection{Hierarchical-Model Evaluation}

Fundamentally, the \emph{hierarchical-model} is a theory of data generation. 
The usefulness of such theories depend on  their predictive 
and explanatory power.

Predictive power is demonstrated  by three examples: A) 
In  \emph{hierarchical-models}, unit-vector-normalized data have  
a previously unknown, statistical maximum pairwise distance of $\sqrt{2}$. This is discussed in \sref{sec:pred} of the supplementary; B) One-class SVM with  traditional normalization  creates very poor results. This is predicted in \sref{sec:re-norm} and demonstrated in \sref{sec:norm_effect} of the  supplementary;
C) Shell-Learning should work, as  demonstrated in \sref{sec:human_shell}.

Explanatory power is demonstrated by three examples: I) Supplementary 
\sref{sec:euc} shows  high dimensional Euclidean distance  remains meaningful.
This differs from our  current paradigm
in which they are ``cursed''~\cite{aggarwal2001surprising}
and provides a mathematical model that accommodates the very high dimensional, yet
very effective
deep-learned features~\cite{arandjelovic2016netvlad,he2016deep, simonyan2014very}; II) In the \emph{hierarchical-model}, small differences in distribution parameters
suffice to  ensure  separability of their instances.  If
deep-learners are exploiting such  properties,
they could be both robust in practice but  easily fooled by deliberate perturbations~\cite{nguyen2015deep}. This creates a plausible explanation for a long-standing puzzle;
III)   \sref{sec:re-norm}  explains 
 how and why  data should  be normalized,
 a concept systematically exploited by our SS algorithm  in \sref{sec:stacked}.
 Previously, normalization's importance  was well documented~\cite{arandjelovic2013all} but there was little  justification  for its effectiveness.  
 
 \renewcommand{\arraystretch}{1.0}
 \renewcommand{\tabcolsep}{1pt}
 \begin{figure}[htp]
 	\small
 	\centering
 	\begin{tabular}{lc}
 		\shortstack{\textbf{Linear-SVM:}\\\\\\\\3 unidentified\\1 women\\8 men}&\includegraphics[width = 0.73\linewidth]{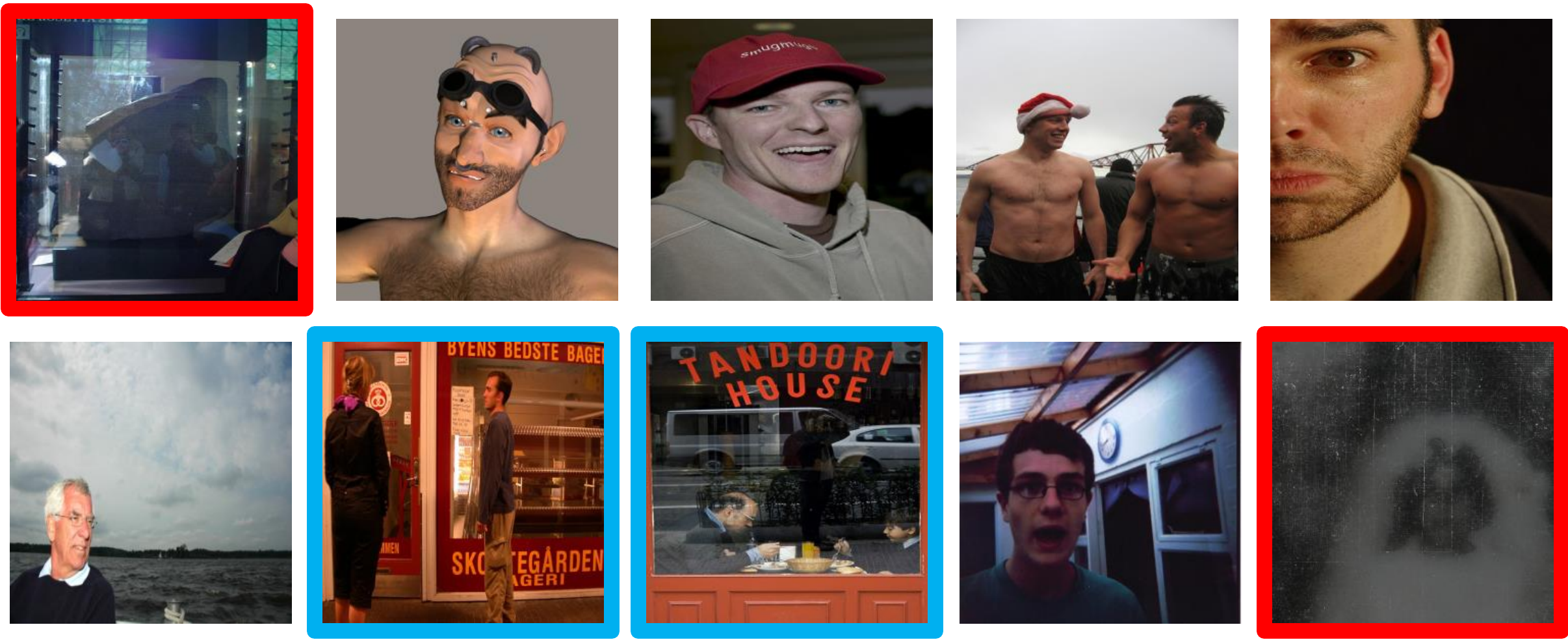}\\%Linear-SVM: 3 un-identified, 1  woman,  and  8 men.\\
 		\midrule
 		\shortstack{\textbf{Shell-One:}\\\\\\\\0 unidentified\\6 women\\9 men}&\includegraphics[width = 0.74\linewidth]{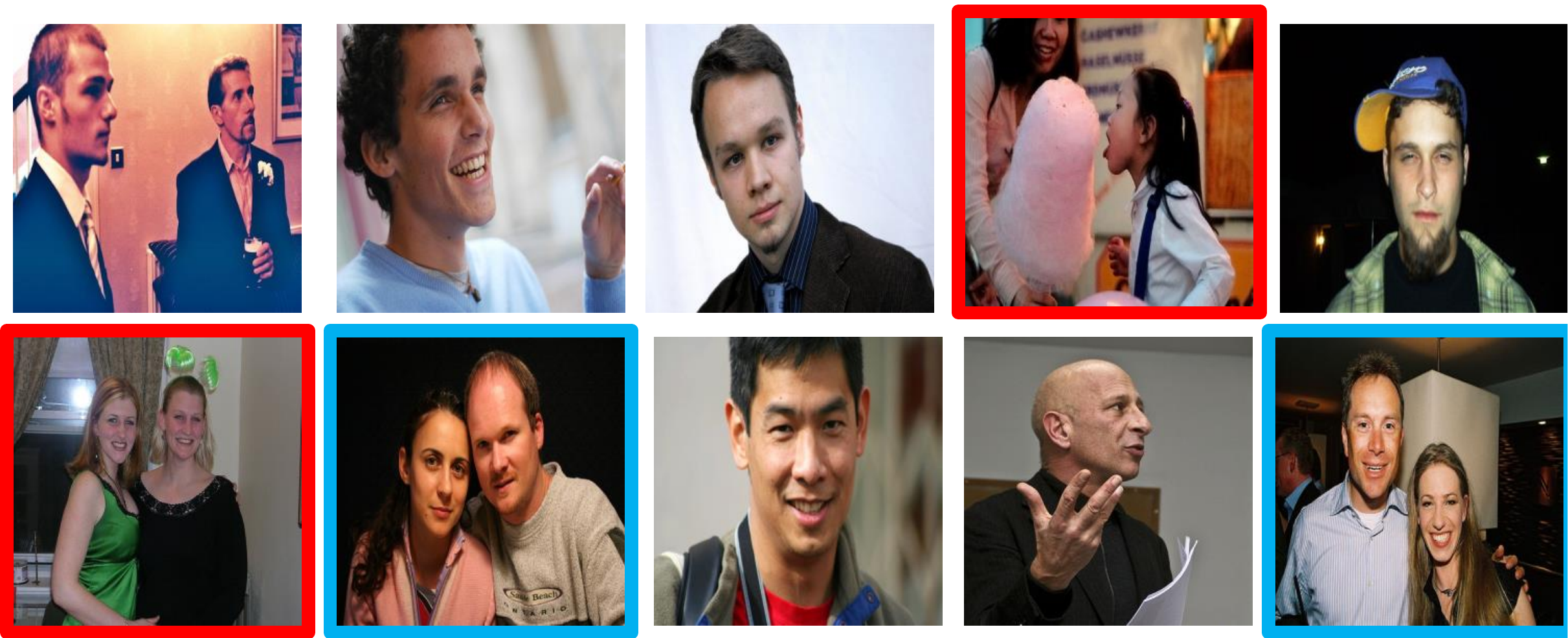}\\
 		%Shell-One:  0 un-identified, 6 women and  9 men.\\
 		\midrule
 		\shortstack{\textbf{Shell-Stack:}\\\\\\\\0 unidentified\\0 women\\11 men} & \includegraphics[width = 0.75\linewidth]{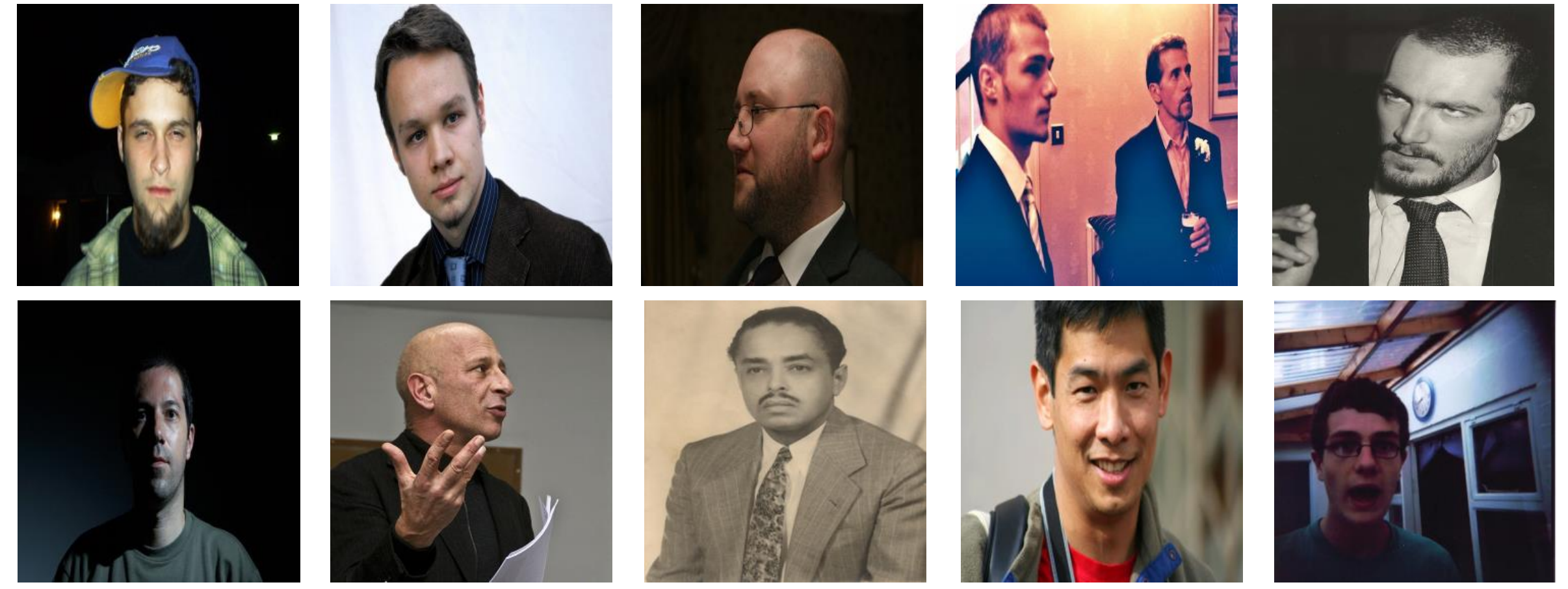}\\
 		%Shell-Stack:  0 un-identified, 0  women and 11 men.
 	\end{tabular}
 	\caption{Learning  to detect   ``man''   in  Flickr11k~\cite{kuo2011unsupervised}. Training images are crawled from the web with semantic search words. Top 10 returns of each algorithm are shown. Errors are  boxed in red,  partial errors in blue.  Discriminative learners like linear-SVM perform
 		poorly on such open-set tasks. Generative shell-learners are  better as they  estimate the probability of a class given the image.\label{fig:ss_human}}
 \end{figure}
 
\subsection{Shell Based Learning}
\label{sec:human_shell}

SO and SS refer to \sref{sec:stacked}'s Shell-One and Shell-Stacked algorithms.
By  strict definitions,  
SO and SS  are not one-class learners as feature choice and normalization embed
assumptions on external distributions.
In particular, SS  treats the means of other training data
as external knowledge for estimating  parent distribution means. 
 However,  given the  omnipresence of the internet, such  assumptions  are not  restrictive. 
% Further,   ``purer'' one-class learners may have made  similar assumptions implicitly.
While comparing algorithm scores   does  not represent fair competition, it  creates a baseline for assessing 
  \emph{hierarchical-model}'s  relevance in  the    real-world.

Data-sets are described by, name; [citation]; (number of training images per class). They are: Fashion-MNIST~\cite{xiao2017fashion} (10,000); STL-10~\cite{coates2011analysis}  (500); Internet STL-10~\cite{coates2011analysis}  (variable); MIT-Places first 5 classes~\cite{zhou2014learning} (1000); ASSIRA~\cite{asirra} (1000). 
STL-10 and ASSIRA represent standard bench-marks for one-class learning. Fashion-MNIST and MIT-Places contain less typical concepts
that features are not pre-trained on. Finally, Internet STL-10 tests
 algorithm's noise handling by having a training set that consist \textbf{solely} of Internet crawled images and a testing set of  STL-10 images.

Baseline algorithms are one-class SVM~\cite{chen2001one} and state-of-art deep-learned alternatives~\cite{zhai2016deep,golan2018deep,zong2018deep,deecke2018anomaly,chalapathy2018anomaly}. 
The deep-learners  typically  couple  an auto-encoder to some cost function,
creating  an end-to-end learning framework that
can modify  all stages for better performance.
In some  architectures, replacing end-to-end learning with   ResNet features improves performance. For such cases,   ResNet features are used as input. 
Results are tabulated in \tref{tab:all}.

SS and SO provide simple learning frameworks that are competitive with or better than far more complicated alternatives.  
SS's performance is so good that when trained solely with internet data in
 Internet STL, its  performance is still better than any  learner trained 
on un-corrupted STL-10 data. 
SS's performance gains are especially pronounced on challenging data. This is illustrated in \tref{tab:hard}, which shows performance on  the 
 most ambiguous class pairs
  in Internet STL-10 and MIT-Places.  
We believe these results indirectly validate the \emph{hierarchical-model} used to design SS and SO.

\subsection{Shell Learners as Classifiers}
\label{sec:precision}

Multiple independently trained shell-learners
can be used to perform multiclass classification
by assigning  each instance to the highest scoring learner.
On this task, 
SS's  average  precision over all  data-sets is
 0.862. This is only slightly below discriminative classifiers like linear-SVM at 0.899
 and much higher than radial-basis-kernel SVM at 0.437, whose density based formulation may be inappropriate for high dimensions where density vanishes.

 These  are remarkable results that, to our knowledge, have never been demonstrated   by  one-class learning formulations~\cite{kardan2016fitted}. \Fref{fig:global} shows  utilizing one-class SVM on the same task creates a vertical zig-zag precision-recall line, corresponding to each learner creating its own unique range of scores. The results are still more remarkable
 when we recall   SS is a  generative formulation, which trades some discriminative performance~\cite{ng2002discriminative} for 
open-set capability demonstrated in \fref{fig:ss_human}. 

%
%
%Thus far, the experiments use traditional relative scoring.
% However, 
%shell-learners  can estimate  absolute  probability. Thus, 
%given multiple shell-learners, an image's class
%can be decided  by assigning it to the one with the highest score.
%SS's  average  precision over all  data-sets is
% 0.862. This is only slightly below discriminative classifiers like linear-SVM at 0.899
% and much higher than radial-basis-kernel SVM at 0.437 (perhaps radial basis are  inappropriate in high dimensions). 
% These are remarkable results, not typically associated with one-class learner, as shown in \fref{fig:global}. Results seem still more remarkable
% when we recall   SS is a  generative formulation, which trades  discriminative performance~\cite{ng2002discriminative} for 
%open-set capability demonstrated in \fref{fig:ss_human}. 
%
%From  \fref{fig:global},  SS   achieves recognition precision  
% that is only slightly below discrimnative classifiers like linear-SVM~\cite{Hearst,fan2008liblinear}
% and  SVC~\cite{platt1999probabilistic} (which may be failing due to inappropriate RBF kernels).  This is a remarkable result that has not been previously demonstrated in one-class learning. 
%Further,  SS is a  generative formulation, which trades some discriminative performance~\cite{ng2002discriminative} for 
%the generalizing ability demonstrated in \fref{fig:ss_human}. 

\section{Conclusion}

This paper suggests many data generation processes can be explained by a \emph{hierarchical-model}.  This  makes it possible to mathematically formulate open-set problems that often seem impossible to analyze rigorously. 
We demonstrate this with a one-class learning formulation that adds new classes to a model without needing retraining. This creates the exciting prospect of life-long learners that 
can expand their understanding indefinitely~\cite{kardan2016fitted}.  
{\small
\bibliographystyle{plain}
\bibliography{egbib_up}
}

%\end{document}

\newpage

\section{Supplementary Material (Theoretical)}

\subsection{Proof for Re-Normalization Corollary}
\label{sec:re-norm_proof}
Before proving Corollary \ref{col:re-norm}, we need to introduce a Lemma: \begin{lemma}
\label{lemma:mean}

A distribution's mean, one of it's sub-distribution's mean and  an arbitrarily point $\bc \in \mathbb{R}^k$  almost surely form a right-angled triangle,
\begin{equation}
d^2(\bM_{\btheta^n\bTheta^j} -\bc) \as d^2(\bmu_{\btheta^n} -\bc) + d^2(\bM_{\btheta^n\bTheta^j}  -\bmu_{\btheta^n})
\end{equation} 
\end{lemma}
\begin{proof}
From \eref{eq:main1} and Theorem \ref{theorem:mvde}, 
\begin{equation}
\label{eq:mean_a}
\begin{split}
&d^2(\bA_{\btheta^n\bTheta^j} -\bc)   \\
\as& V_{\btheta^n\bTheta^j} + d^2(\bM_{\btheta^n\bTheta^j} -\bc)\\
\as & v_{\btheta^{[n]}} - d^2(\bM_{\btheta^n\bTheta^j}-\bmu_{\btheta^{[n]}})   + d^2(\bM_{\btheta^n\bTheta^j} -\bc)
\end{split}
\end{equation}

\begin{equation}
\label{eq:mean_b}
d^2(\bA_{\btheta^n\bTheta^j} -\bc) \as d^2(\bA_{\btheta^n} -\bc) \as   v_{\btheta^{[n]}} +d^2(\mu_{\btheta^{[n]}} -\bc)
\end{equation}

Combining \eref{eq:mean_a}, \eref{eq:mean_b}, yields
$$d^2(\bM_{\btheta^n\bTheta^j} -\bc) \as d(\bmu_{\btheta^n} -\bc) + d^2(\bM_{\btheta^n\bTheta^j}  -\bmu_{\btheta^n})
$$
\end{proof}

Using Lemma \ref{lemma:mean}, the proof for the Re-Normalization Corollary \ref{col:re-norm} is:
\begin{proof}
If  $l$ is an  ancestor of both $m,n$, it can be considered a distribution-of-everything encompassing both sub-distributions. Thus,  re-normalizing with   mean $\bmu_{\btheta^l}$
only involves a translation and division by a  scalar common to both $m$ and $n$.
Thus, from \eref{eq:global_norm}, $\widetilde{g_{mn}} = \frac{2(v_{\btheta^{m}}- v_{\btheta^{n}})}{v_{\btheta^{l}}}$, proving the first case of \eref{eq:re-norm}.

For the case  
$l \leq m$,  with  $\btheta^m$ representing a parent of $\btheta^l$.
 After 
re-normalization with $\bmu_{\btheta^l}$,  almost-surely:
\begin{equation}
\label{eq:1_const}
\begin{split}
\bmu_{\widetilde{\btheta^l}} &=\mathbf{0};\\
v_{\widetilde{\btheta^n}} + d^2(\bmu_{\widetilde{\btheta^n}}) &= 1; \quad \mbox{because } d^2(\widetilde{\bA}_{\widetilde{\btheta^n}} - \mathbf{0}) \as 1 \\
v_{\widetilde{\btheta^m}} + d^2(\bmu_{\widetilde{\btheta^m}}) &= 1; \quad \mbox{because } d^2(\widetilde{\bA}_{\widetilde{\btheta^m}} - \mathbf{0}) \as 1 \\
\end{split}
\end{equation}

Replacing $\bc$  in Lemma \ref{lemma:mean}, with  $ \bmu_{\widetilde{\btheta^m}}$,
\begin{equation}
\label{eq:mean_apply}
d^2(\bM_{\widetilde{\theta^l}\bTheta^j} - \bmu_{\widetilde{\btheta^m}}) \as d^2(\bmu_{\widetilde{\btheta^m}}-\bmu_{\widetilde{\btheta^l}}) + d^2(\bM_{\widetilde{\theta^l}\bTheta^j}-\bmu_{\widetilde{\btheta^l}}).
\end{equation}
As $\bmu_{\widetilde{\btheta^n}}$
is an instance of $\bM_{\widetilde{\theta^l}\bTheta^j}$,  \eref{eq:1_const}
and \eref{eq:mean_apply}, almost-surely result in 
\begin{equation}
\label{eq:mean_apply1}
\begin{split}
d^2(\bmu_{\widetilde{\btheta^n}} - \bmu_{\widetilde{\btheta^m}}) &= d^2(\bmu_{\widetilde{\btheta^m}}-\bmu_{\widetilde{\btheta^l}}) + d^2(\bmu_{\widetilde{\btheta^n}}-\bmu_{\widetilde{\btheta^l}})\\
&= d^2(\bmu_{\widetilde{\btheta^m}}) + d^2(\bmu_{\widetilde{\btheta^n}})
\end{split}
\end{equation}

Combining \eref{eq:main1}, \eref{eq:1_const}, \eref{eq:mean_apply1}, almost-surely  
\begin{equation}
\begin{split}
d^2(\widetilde{\bA}_{\widetilde{\btheta^m}\bTheta^i} - \bmu_{\widetilde{\btheta^n}}) 
&\as v_{\widetilde{\btheta^m}} + d^2(\bmu_{\widetilde{\btheta^m}} -\bmu_{\widetilde{\btheta^n}})\\
&= v_{\widetilde{\btheta^m}} + d^2(\bmu_{\widetilde{\btheta^m}}) + d(\bmu_{\widetilde{\btheta^n}})\\
&= 2- v_{\widetilde{\btheta^n}}
\end{split}
\end{equation}
This proves the last case in \eref{eq:re-norm}, 
\begin{equation}
\widetilde{g_{mn}} = 2 (1- v_{\widetilde{\btheta^n}}) = \frac{2(v_{\btheta^l}-v_{\btheta^n})}{v_{\btheta^l}}
\end{equation}

\end{proof}

\subsection{Euclidean Distance in High Dimensions}
\label{sec:euc}
In the \emph{hierarchical-model}, the Euclidean distance between two high dimension
instances is almost-surely a reflection of how recently they shared a common ancestor. 
This is summarized in the Euclidean distance  corollary below:

\begin{corollary}
\label{col:nn}

(Euclidean distance) In a non-trivial
 \emph{hierarchical-model}, ranking data points by euclidean distance from $\ba$, is almost surely,
ranking  based on how recently they shared an ancestor with $\ba$.

\end{corollary}
\begin{proof}
Let $\ba$ be an instance of $A_{\btheta^m}$, where $$\btheta^m = \btheta^{[0]}\btheta^{[1]}\btheta^{[2]}\hdots \btheta^{[m]},$$ is  an instantiation of a non-trivial hierarchical process.
 The average variance  corresponding to the parameters $\btheta^{[i]}$ is $v_{\btheta[i]}$. From  corollary \ref{col:v}, almost surely,  $$v_{\btheta[0]} >v_{\btheta[1]} >v_{\btheta[2]} > \hdots >v_{\btheta[m]}.$$

For any data-point $\mathbf{p}$ in the \emph{hierarchical-model},  Theorem \ref{theorem:L2} implies, almost-surely,
$d^2(\ba-\mathbf{p}) = 2v_{\btheta[i]}$, where $i$ is the index of their most recent common ancestor.  
Thus, ranking based euclidean distance from $\ba$; is almost-surely, ranking based on average
variance of the most recent ancestor; is almost-surely,  ranking based on how
recently the points shared an ancestor. 
\end{proof}

\subsection{Predicted Distance Histograms}
\label{sec:pred}

This section analyzes distances with deep-learned features to
see if they follow the patterns predicted by the \emph{hierarchical-model}.
Results for real and simulated data are displayed  in \fref{fig:hist}.

\textbf{A) Statistical maximum pair-wise distance is  $\sqrt{2}$:}
Recall that for 
 unit-vector-normalized data, $$d^2(\bA-\mathbf{0}) \as  v_{\bA} + d^2(\bmu_{\bA}-\mathbf{0}) = 1.$$ Thus, $v_\bA \leq 1$ 
From \fref{fig:tree} and  Corollary \ref{col:nn},   the distance between any two instances is almost-surely less than or equal to  $\sqrt{2v_\bA}$ which is in turn less than or equal to  $\sqrt{2}$. 
This creates a statistical maximum pairwise distance is  $\sqrt{2}$ that is 
much less than the geometric maximum of $2$.

\textbf{B) Distance to random unit-vector almost-surely constant:}

From \eref{eq:main1}
\begin{equation}
d^2(\bA_{\bTheta^i}-\bc) \as v_{\bA} + d(\bmu_\bA -\bc) = \mbox{constant}.
\end{equation}
This equation has often been mis-interpreted as a ``curse''. While this is true, a different
pattern emerges when considering pairwise distances;

\textbf{C) Log  pairwise histogram shows a variety of distances:} 
From theorem \ref{theorem:L2} and \fref{fig:tree}, not all pairwise distances are constants. 
While many pairwise distances  converge to $\sqrt{2v_\bA}$, the log  histogram displays a variety of distances.

\begin{figure*}
  \begin{minipage}[c]{0.45\textwidth}
  \begin{tabular}{c}
 \includegraphics[width = 0.6\linewidth]{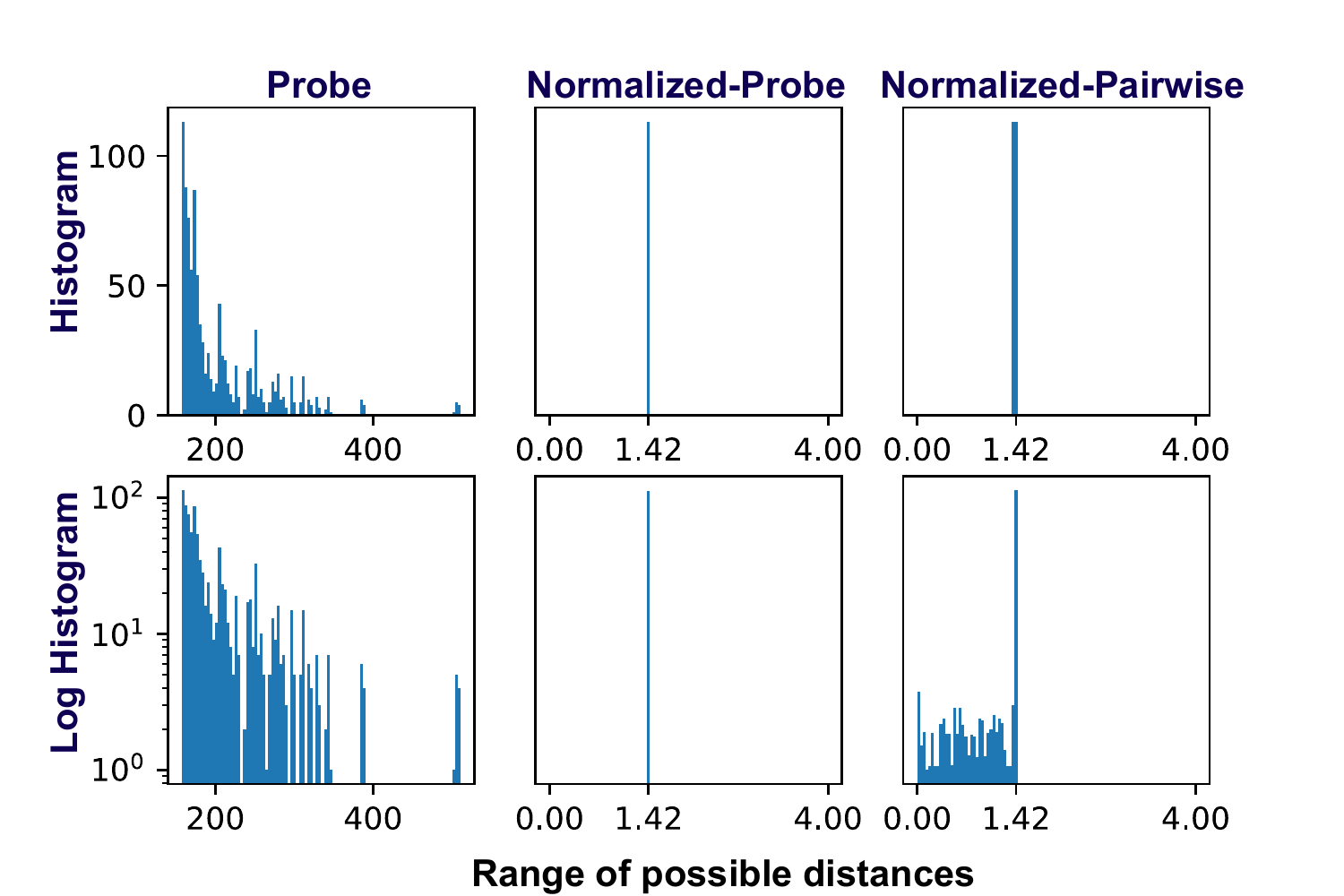}\\
 \includegraphics[width = 0.6\linewidth]{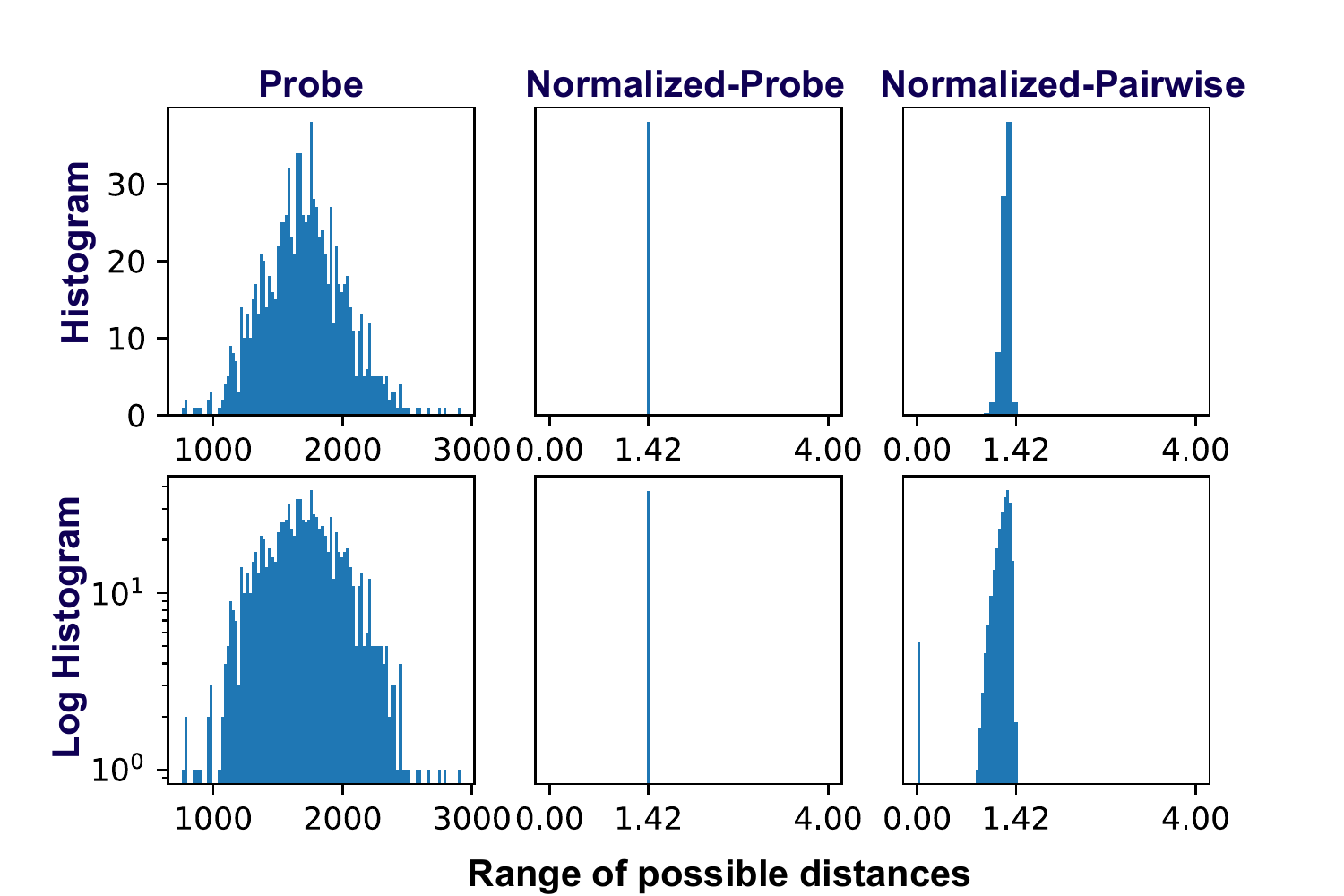}
 \end{tabular}
  \end{minipage}
   \hskip -2cm
  \begin{minipage}[c]{0.65\textwidth}
  \caption{
   \textbf{Probe} is the histogram of the distance of all data-set instances to a single, unit vector probe. Data derived from a hierarchical-process, will have a histogram sharply peaked at some point.  This is clearly not the case for both simulated and real-data. \textbf{Normalized-Probe} involves unit-vector-normalization of all data  before computing their distance to the probe.  Data  derived from a hierarchical-process perturbed by some scaling functions will have  histograms which  peak sharply at $\sqrt{2}=1.42$. This happens for both simulated and real-data. \textbf{Normalized-Pairwise} is the histogram of pair-wise distance between data-set points. Data derived from complex hierarchical-processes, have histograms that  peak sharply  at $1.42$ but the log histogram should show some instances that are less than $1.42$. $1.42$ is the statistical maximum distance which almost no distance should exceed.  \textbf{Summary:}  A hierarchical-process with scalar perturbations is an eerily accurate model of real-world data. This is reflected in histograms which exhibit the surprising characteristics predicted by the hierarchical-processes and in our ability to create simulations   that  mimic these characteristics. }
  \label{fig:hist}
  \end{minipage}
\end{figure*}

\newpage

\section{Supplementary Material (Empirical)}

\subsection{Errors}
\Fref{fig:shell_dist} shows distances of instances from a shell's center. Labels
are color coded. The shell is trained with instances of the red class. 
Observe that in practice,  distances do not lie and a perfect constant but form a
 a cloud. Despite this, points of  the red class are separable from the rest. Separability can be increased through re-normalization in  \sref{sec:re-norm}. 

\begin{figure}
\centering
\includegraphics[width = 0.49\linewidth]{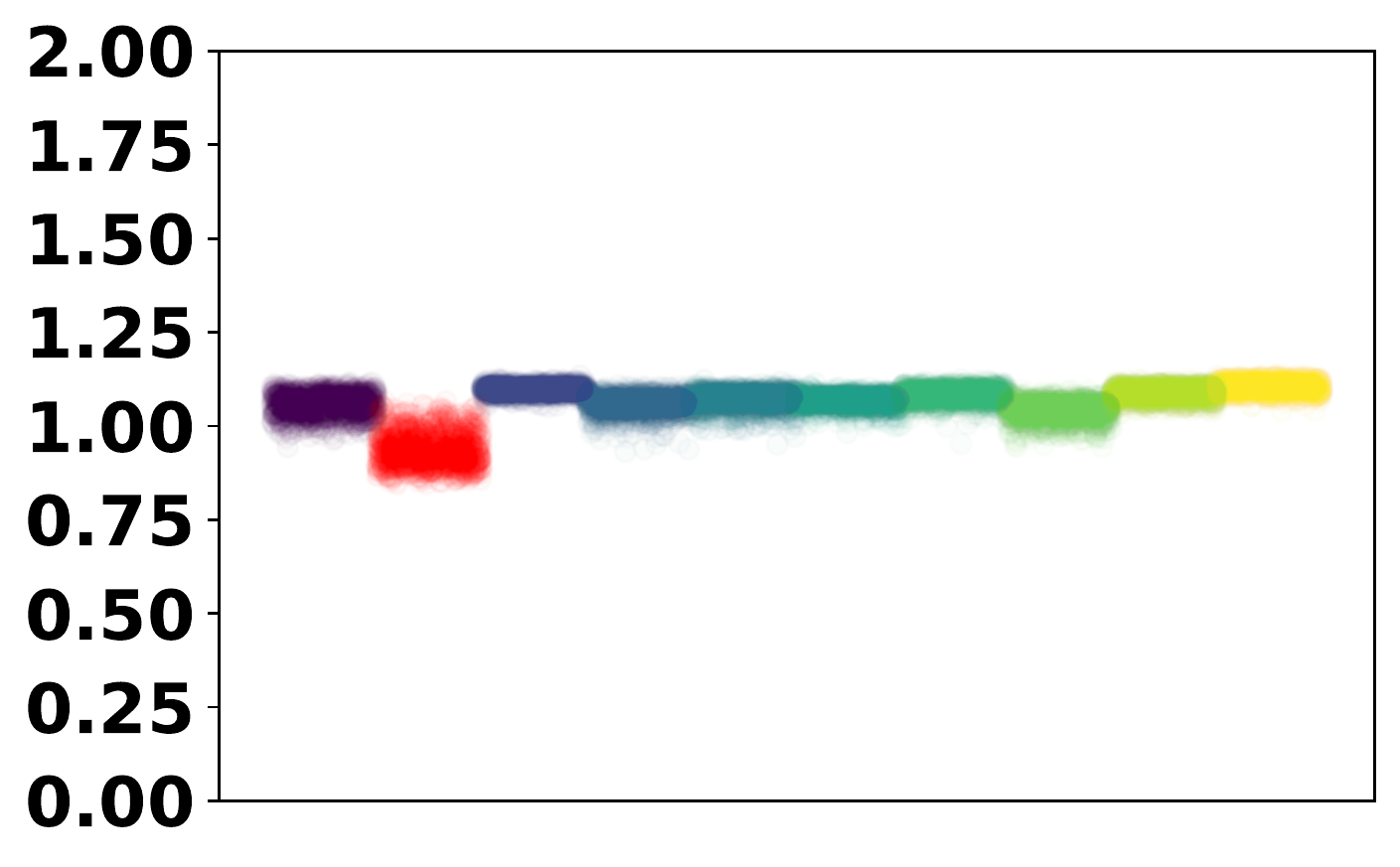}
\includegraphics[width = 0.49\linewidth]{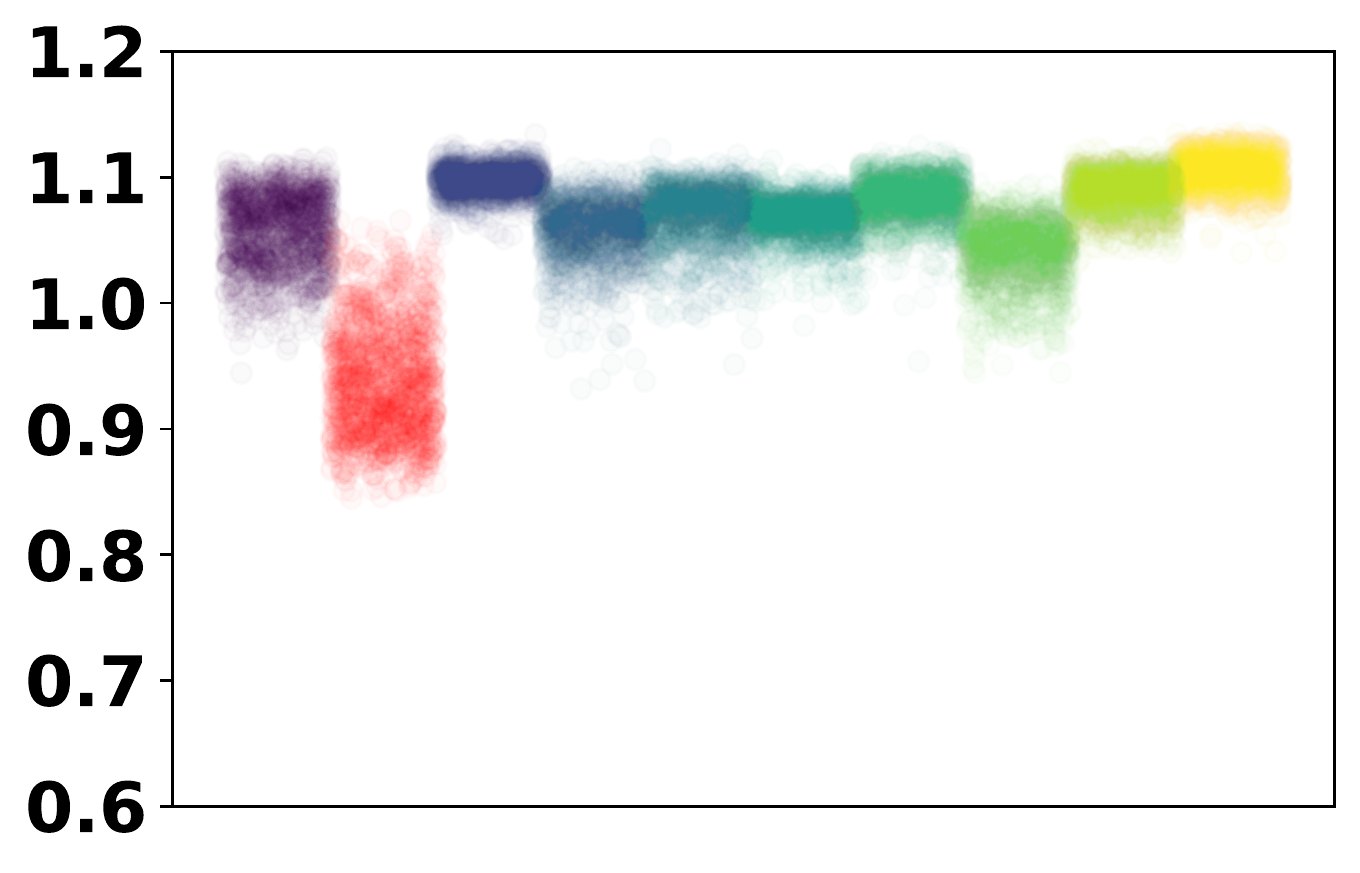}
\caption{Distances to shell center  for  the ``bird'' class in STL-10~\cite{coates2011analysis}. Colors represent different classes of STL-10 dataset. Red color represents the 'bird' class to which the shell is fitted.\label{fig:shell_dist}}
\end{figure}

\subsection{Normalization in One-Class Learning}
\label{sec:norm_effect}

To date, there has been no clear guide on how to perform normalization in one-class learning. \Tref{tab:norm} shows  traditional normalization which involves
 centering data and converting it to unit-vectors causes extremely poor results.
 This was predicted  in \ref{sec:re-norm}. As a result one-class learning papers
  perform no normalization. 
 \tref{tab:norm}.  
%%% SO vs OC-SVM with/without Unit-normalization
\begin{table*}
\centering
\begin{tabular}{lccccccccccc}
\toprule
&\multicolumn{11}{c}{STL-10~\cite{coates2011analysis} (ResNet Features~\cite{he2016deep})} \\
\hline
&airplane & bird & car & cat & deer & dog & horse & monkey & ship & truck & \textit{Ave.}\\
\midrule
OC-SVM(no-normalization) & 0.854 & 0.748 & 0.949 & 0.689 & 0.857 & 0.553 & 0.792 & 0.709 & 0.929 & 0.905 & 0.799 \\
%SO & 0.924 & 0.923 & 0.960 & 0.878 & 0.950 & 0.864 & 0.934 & 0.911 & 0.960 & 0.915 & 0.922 \\
OC-SVM(traditional-normalization) & 0.550 & 0.458 & 0.670 & 0.468 & 0.492 & 0.351 & 0.565 & 0.471 & 0.722 & 0.472 & 0.522 \\
%SO(extract-mean) & 0.900 & 0.861 & 0.914 & 0.764 & 0.862 & 0.878 & 0.848 & 0.835 & 0.903 & 0.889 & 0.865 \\
%OC-SVM(U-norm) & 0.974 & 0.965 & 0.985 & 0.894 & 0.962 & 0.928 & 0.960 & 0.962 & 0.983 & 0.969 & 0.958 \\
%SO(U-norm) & 0.977 & 0.964 & 0.984 & 0.882 & 0.958 & 0.906 & 0.953 & 0.962 & 0.982 & 0.957 & 0.953 \\
\hline
\hline
\end{tabular}
\caption{
AUROC score of OC-SVM~\cite{chen2001one}  on STL-10. As predicted in \sref{sec:re-norm},
traditional normalization makes   one-class learning very much worse.  
As such normalization is seldom perfomed. \label{tab:norm}}
\end{table*}

\subsection{Shell-Learning vs One-Class SVM}
\label{sec:mean}

Given properly normalized data, shell-learning (SO) is similar to one-class SVM. However,
there is a significant difference in more difficult conditions. 
\Fref{fig:mean} shows the effect of shifting the mean used for normalization towards a target class, the degenerate case given in \eref{eq:re-norm}.  Observe that while 
shell-learnings performance decreases, it remains much more stable than one-class SVM. 

\begin{figure}
\centering
\includegraphics[scale=0.4]{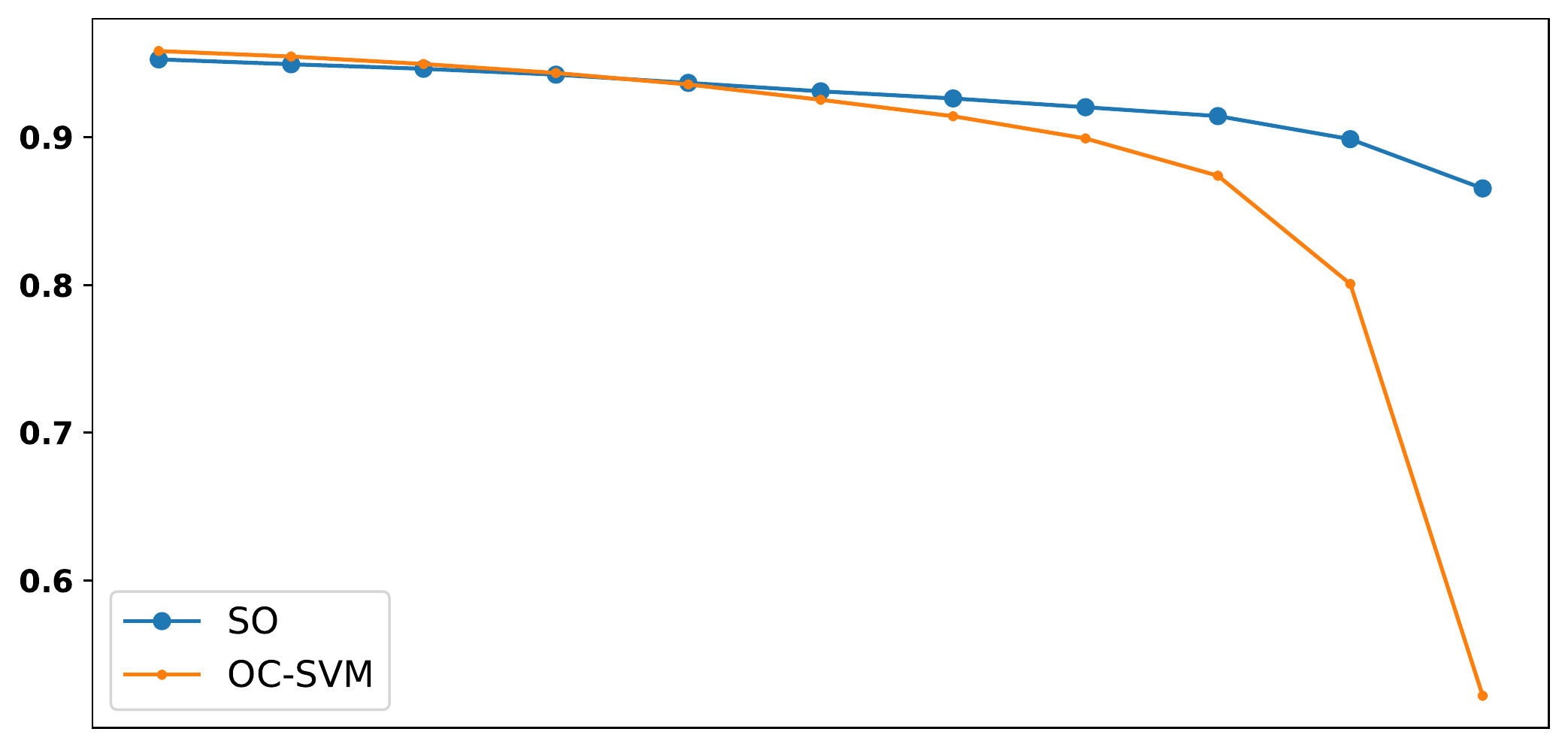}
\caption{Average AUROC score on STL-10 dataset, with increasingly  poor normalization. 
Observe that shell-learning (SO) is much more robust than traditonal one-class SVM~\cite{chen2001one} to poor normalization. \label{fig:mean}}
\end{figure}

\subsection{Magic of Deep Learning Features}
\label{sec:magic_sup}
Applying one-class SVM on deep-learned features creates magically good results. 
Much of this result is because deep-learned features use a coordinate frame in which zero corresponds to the mean of the distribution-of-everything. Shifting 
the coordinate frame causes a marked deterioration in the results 
as shown in 
\fref{fig:dl_mean}. 
\begin{figure}
\centering
\includegraphics[scale=0.4]{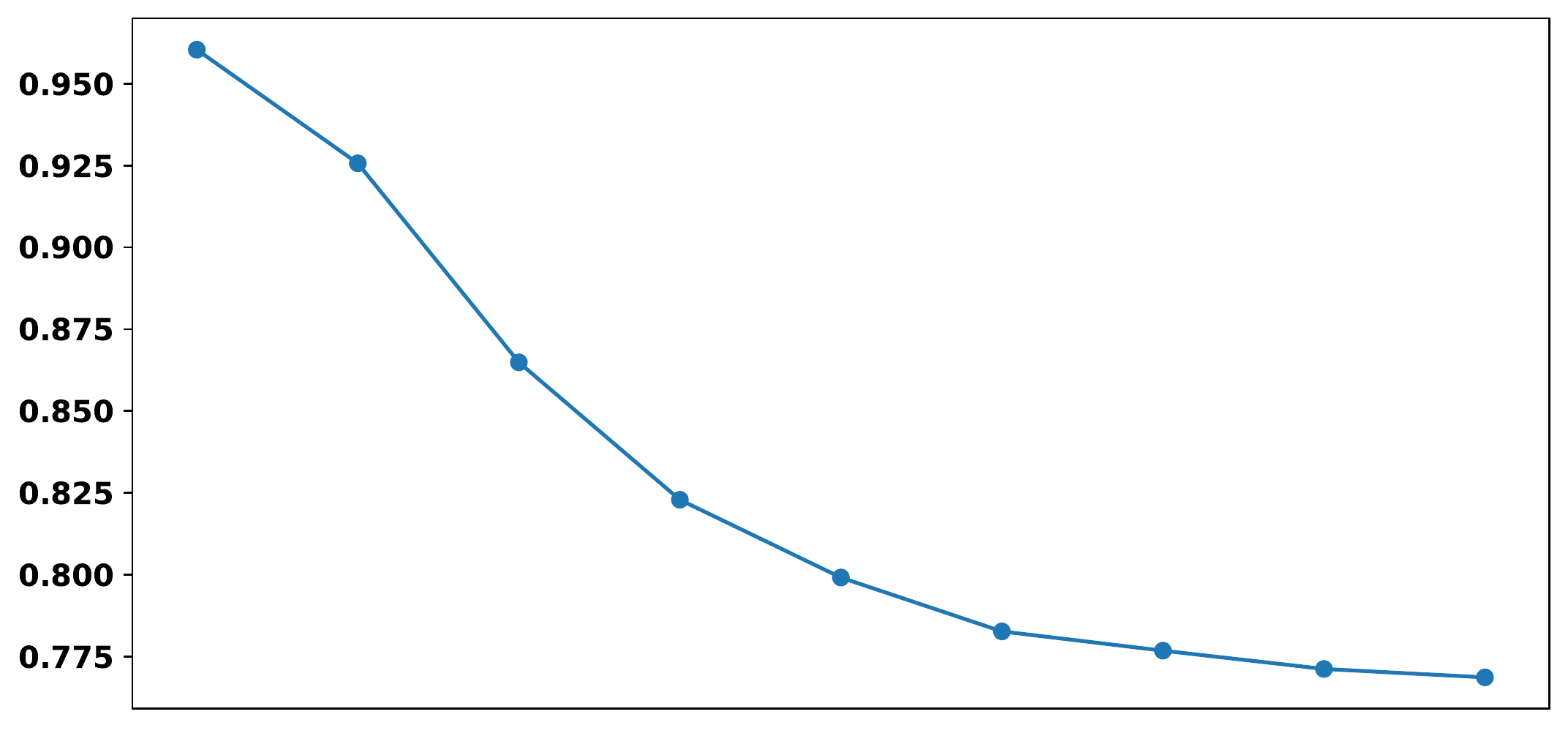}
\caption{Average AUROC score on STL-10 dataset using ResNet~\cite{he2016deep} and OC-SVM~\cite{chen2001one}. Observe that much of the magically good performance of ResNet~\cite{he2016deep} derives from its choice of coordinate frame. Shifting the mean  from origin caueses a steady  detotriation in results. \label{fig:dl_mean}}
\end{figure}

\subsection{Effect of Shells-Stacked (SS)}
\label{sec:stack_effect}

\Fref{fig:stack_effect}  plots the one-class detection AUROC score on STL-10 dataset, with increasingly
stacked shells. This shows the effectiveness of re-normalization in incorporating external knowledge into a one-class learning framework. 
\begin{figure}[htp]
\centering
\includegraphics[scale=0.4]{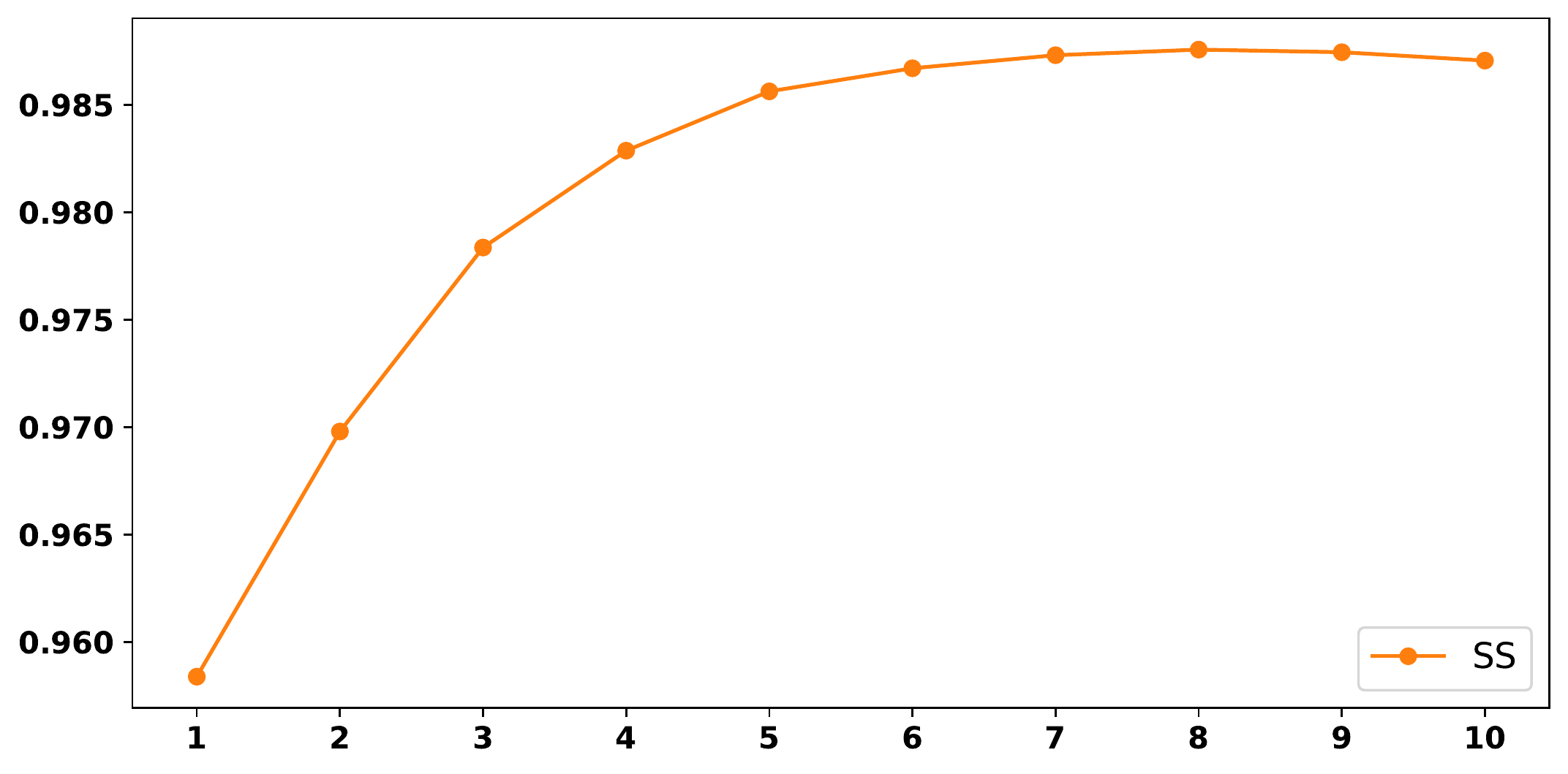}
\caption{AUROC score on STL-10 dataset for  Shells-Stacked (SS)  with increasing
number of shells. This demnonstartes it ability to leverage external information. \label{fig:stack_effect}}
\end{figure}

\subsection{Mean-Variance Constraint Predicted by Hierarchical-Model}
Sec. \ref{subsec:parameter_constraint} shows that the parent and child distributions generated by a \emph{hierarchical process} in the \emph{hierarchical-model} must adhere to Theorem \ref{theorem:mvde}. Results in \tref{tab:mvde} shows that this relationship in real-world data follows the mean-variance constraint predicted by our \emph{hierarchical-model}. Statistics for various fine-grained types of dogs and cats are collected from ImageNet data under Domestic Animal-(Domestic Dog, Domestic Cat) hierachical synsets. Animals and transports data are from STL-10 dataset. Images are converted into ResNet features using pretrained ResNet50 and unit-normalized. Variances and means are calculated based on unit-normalized ResNet features. Low error ratio indicates the mean-variance relationship				 of parent and child distributions in real-world data is consistent with Theorem \ref{theorem:mvde}.  
%%% Hierachical-Model Mean-Variance Cosntraints Validation
\begin{table*}
\centering
\begin{tabular}{c|c|c|c}
\toprule
\makecell{\\ Parent Distribution \\ $\bA_{\btheta^n}$} & \makecell{Average Variance \\  of Parent Distribution \\ $v_{\btheta^{[n]}}$} & \makecell{Mean Average Variance of Sub-Distribution + \\ NSD of Mean Vectors for all Sub-Distributions\\ $V_{\btheta^n\bTheta^j}$ + $\dif(\bM_{\btheta^n\bTheta^j}-\bmu_{\btheta^{[n]}})$} & \makecell{Error \\Ratio \\\%} \\
\hline
Domestic Dog & 0.000308 & 0.000309 & 0.039\% \\
Poodle Dog & 0.000260 & 0.000264 & 1.794\% \\
Spitz & 0.000253 & 0.000259 & 2.363\% \\
Shepherd Dog & 0.000300 & 0.000300 & 0.289\% \\
Sled Dog & 0.000235 & 0.000245 & 4.113\% \\
Watch Dog & 0.000298 & 0.000303 & 1.857\% \\
Sennenhunde & 0.000229 & 0.000229 & 0.041\% \\
Working Dog & 0.000306 & 0.000306 & 0.125\% \\
Domestic Cat & 0.000269 & 0.000268 & 0.119\% \\
% Domestic Animal & 0.000308 & 0.000331 & 7.342\% \\
Animals & 0.000337 & 0.000337 & 0.000\% \\
Transports & 0.000336 & 0.000336 & 0.000\% \\
\hline
\hline
\end{tabular}
\caption{Average variance of parent distribution, and predicted average varaince of parent distribution using Theorem 2, and the respective error ratio in percentage. \label{tab:mvde}}
\end{table*}

\subsection{Detailed Precision Scores for Each Class in Different Datasets}
We provide detailed numbers of multi-class recognition precision scores for each class in all the datasets we tested in \tref{tab:precision}.
%%% Multi-class precision table
\begin{table*}
\centering
\begin{tabular}{lccccccccccc}
\toprule
&\multicolumn{11}{c}{A) Fashion-MNIST (Raw Image)} \\
\hline
& t-shirt & trouser & pullover & dress & coat & sandal & shirt & sneaker & bag & boot & \textit{Ave.} \\
\midrule
SS-Ours & 0.802 & 0.996 & 0.565 & 0.718 & 0.508 & 0.801 & 0.292 & 0.905 & 0.924 & 0.913 & 0.742 \\
SVM(Linear) & 0.812 & 0.968 & 0.713 & 0.799 & 0.691 & 0.951 & 0.669 & 0.901 & 0.912 & 0.920 & 0.834 \\
SVM(RBF Kernel) & 0.000 & 0.295 & 0.081 & 0.000 & 0.000 & 0.000 & 0.197 & 0.468 & 0.000 & 0.721 & 0.176 \\

\hline
\hline
\end{tabular}

\centering
\begin{tabular}{lccccccccccc}
&\multicolumn{11}{c}{B) STL-10~\cite{coates2011analysis} (ResNet-50 Features)} \\
\hline
&airplane & bird & car & cat & deer & dog & horse & monkey & ship & truck & \textit{Ave.}\\
\midrule
SS-Ours & 0.922 & 0.969 & 0.961 & 0.779 & 0.891 & 0.729 & 0.861 & 0.938 & 0.894 & 0.803 & 0.875 \\
SVM(Linear)   & 0.935 & 0.954 & 0.951 & 0.885 & 0.895 & 0.862 & 0.903 & 0.909 & 0.909 & 0.888 & 0.909 \\
SVM(RBF Kernel) & 0.958 & 0.948 & 0.927 & 0.908 & 0.794 & 0.659 & 0.950 & 0.909 & 0.805 & 0.813 & 0.867 \\
\hline
\hline
\end{tabular}

\begin{tabular}{lccccccccccc}
&\multicolumn{11}{c}{C) Internet STL-10 (ResNet-50 Features)} \\
\hline
&airplane & bird & car & cat & deer & dog & horse & monkey & ship & truck & \textit{Ave.}\\
\midrule
SS-Ours & 0.646 & 0.626 & 0.941 & 0.910 & 0.706 & 0.572 & 0.936 & 0.964 & 0.857 & 0.870 & 0.803 \\
SVM(Linear)   & 0.780 & 0.816 & 0.873 & 0.877 & 0.687 & 0.722 & 0.934 & 0.936 & 0.798 & 0.850 & 0.827 \\
SVM(RBF Kernel) & 0.000 & 0.100 & 0.000 & 0.000 & 0.000 & 0.000 & 0.000 & 0.000 & 0.000 & 0.000 & 0.010 \\
\hline
\hline
\end{tabular}

\begin{tabular}{lcccccc|ccc}
&\multicolumn{6}{c}{D) MIT-Places~\cite{zhou2014learning}}&\multicolumn{3}{c}{E) Assira~\cite{asirra}} \\
\hline
&abbey & \makecell{airport\\ terminal} & alley & \makecell{amusement \\park} & aquarium & \textit{Ave.} & cat & dog &\textit{Ave.} \\
\midrule
SS-Ours & 0.955 & 0.896 & 0.919 & 0.772 & 0.955 & 0.899 & 0.995 & 0.971 & 0.983 \\
SVM(Linear) & 0.936 & 0.941 & 0.946 & 0.911 & 0.962 & 0.939 & 0.989 & 0.990 & 0.990 \\
SVM(RBF Kernel) & 0.945 & 0.905 & 0.915 & 0.723 & 0.924 & 0.882 & 0.000 & 0.500 & 0.250 \\
\hline
\end{tabular}
\caption{Precision score of various multi-class classifiers. Integration of multiple SS as a multi-class classifier achieves recognition precision only slightly below Linear SVM in most cases. \label{tab:precision}}
\end{table*}

\subsection{Detailed AUROC Scores for Each Class in Different Datasets}
We provide detailed numbers of one-class detection AUROC scores for each class in all the datasets we tested in \tref{tab:all_auroc}.
\begin{table*}
\centering
\begin{tabular}{lccccccccccc}
\toprule
&\multicolumn{11}{c}{A) Fashion-MNIST (Raw Image)} \\
\hline
& t-shirt & trouser & pullover & dress & coat & sandal & shirt & sneaker & bag & boot & \textit{Ave.} \\
\midrule
OC-SVM~\cite{chen2001one} & 0.884 & 0.972 & 0.857 & 0.900 & 0.876 & 0.868 & 0.787 & 0.982 & 0.816 & 0.976 & 0.892 \\
SO-Ours & 0.911 & 0.975 & 0.881 & 0.912 & 0.904 & 0.921 & 0.772 & 0.986 & 0.856 & 0.990 & 0.911 \\
SS-Ours & 0.954 & 0.988 & 0.921 & 0.967 & 0.931 & 0.975 & 0.830 & 0.989 & 0.981 & 0.991 & 0.953 \\
OC-NN & 0.818 & 0.963 & 0.739 & 0.917 & 0.916 & 0.932 & 0.678 & 0.948 & 0.614 & 0.982 & 0.851 \\
Deep A.Det~\cite{golan2018deep} & 0.918 & 0.981 & 0.916 & 0.879 & 0.898 & 0.934 & 0.827 & 0.991 & 0.987 & 0.994 & 0.932 \\
DSEBM & 0.916 & 0.718 & 0.883 & 0.873 & 0.852 & 0.871 & 0.734 & 0.981 & 0.860 & 0.971 & 0.884 \\
DAGMM & 0.421 & 0.551 & 0.504 & 0.571 & 0.269 & 0.705 & 0.483 & 0.835 & 0.499 & 0.340 & 0.518 \\
AD-GAN & 0.899 & 0.819 & 0.876 & 0.912 & 0.865 & 0.896 & 0.743 & 0.972 & 0.890 & 0.971 & 0.884 \\
\hline
\hline
\end{tabular}

\centering
\begin{tabular}{lccccccccccc}
&\multicolumn{11}{c}{B) STL-10~\cite{coates2011analysis} (ResNet-50 Features)} \\
\hline
&airplane & bird & car & cat & deer & dog & horse & monkey & ship & truck & \textit{Ave.}\\
\midrule
OC-SVM & 0.854 & 0.748 & 0.949 & 0.689 & 0.857 & 0.553 & 0.792 & 0.709 & 0.929 & 0.905 & 0.799 \\
% OCSVM(unit-norm to class-mean) & 0.555 & 0.455 & 0.656 & 0.603 & 0.411 & 0.552 & 0.272 & 0.639 & 0.639 & 0.678 & 0.546 \\
SO-Ours & 0.974 & 0.965 & 0.985 & 0.894 & 0.962 & 0.928 & 0.960 & 0.962 & 0.983 & 0.969 & 0.958 \\
SS-Ours & 0.993 & 0.994 & 0.993 & 0.976 & 0.981 & 0.978 & 0.984 & 0.992 & 0.993 & 0.987 & 0.987 \\
OC-NN & 0.973 & 0.962 & 0.969 & 0.881 & 0.955 & 0.894 & 0.952 & 0.966 & 0.972 & 0.971 & 0.949 \\
Deep A.Det. & 0.575 & 0.544 & 0.801 & 0.632 & 0.819 & 0.725 & 0.841 & 0.858 & 0.797 & 0.706 & 0.730 \\
DSEBM & 0.677 & 0.590 & 0.425 & 0.630 & 0.735 & 0.569 & 0.534 & 0.639 & 0.607 & 0.308 & 0.571 \\
DAGMM & 0.729 & 0.503 & 0.613 & 0.596 & 0.647 & 0.479 & 0.591 & 0.519 & 0.235 & 0.627 & 0.554 \\
AD-GAN & 0.751 & 0.570 & 0.480 & 0.600 & 0.772 & 0.623 & 0.508 & 0.582 & 0.671 & 0.465 & 0.602 \\
\hline
\hline
\end{tabular}

\begin{tabular}{lccccccccccc}
&\multicolumn{11}{c}{C) Internet STL-10 (ResNet-50 Features)} \\
\hline
&airplane & bird & car & cat & deer & dog & horse & monkey & ship & truck & \textit{Ave.}\\
\midrule
%OCSVM(unit-norm to class-mean) & 0.432 & 0.567 & 0.596 & 0.536 & 0.643 & 0.664 & 0.573 & 0.489 & 0.498 & 0.371 & 0.537 \\
OC-SVM & 0.733 & 0.403 & 0.895 & 0.442 & 0.464 & 0.246 & 0.419 & 0.275 & 0.860 & 0.835 & 0.557 \\
SO-Ours & 0.959 & 0.971 & 0.977 & 0.876 & 0.965 & 0.922 & 0.946 & 0.960 & 0.980 & 0.926 & 0.948 \\
SS-Ours & 0.983 & 0.990 & 0.989 & 0.923 & 0.982 & 0.973 & 0.976 & 0.987 & 0.990 & 0.955 & 0.975 \\
OC-NN & 0.939 & 0.955 & 0.967 & 0.851 & 0.949 & 0.891 & 0.938 & 0.941 & 0.968 & 0.921 & 0.932 \\
Deep A.Det. & 0.528 & 0.616 & 0.834 & 0.697 & 0.774 & 0.695 & 0.844 & 0.803 & 0.695 & 0.686 & 0.717 \\
DSEBM & 0.743 & 0.582 & 0.272 & 0.582 & 0.750 & 0.571 & 0.548 & 0.633 & 0.633 & 0.289 & 0.560 \\
DAGMM & 0.607 & 0.500 & 0.486 & 0.481 & 0.564 & 0.530 & 0.485 & 0.619 & 0.408 & 0.489 & 0.517 \\
AD-GAN & 0.743 & 0.533 & 0.406 & 0.550 & 0.734 & 0.558 & 0.501 & 0.604 & 0.618 & 0.306 & 0.555 \\

\hline
\hline
\end{tabular}

\begin{tabular}{lcccccc|ccc}
&\multicolumn{6}{c}{D) MIT-Places~\cite{zhou2014learning}}&\multicolumn{3}{c}{E) Assira~\cite{asirra}} \\
\hline
&abbey & \makecell{airport\\ terminal} & alley & \makecell{amusement \\park} & aquarium & \textit{Ave.} & cat & dog &\textit{Ave.} \\
\midrule
OC-SVM & 0.880 & 0.800 & 0.865 & 0.630 & 0.650  & 0.765 & 0.899 & 0.750 & 0.824\\
%OCSVM(unit-norm to class-mean) & 0.487 & 0.691 & 0.673 & 0.489 & 0.424 & 0.553  & 0.678 & 0.431 & 0.555 \\
SO-Ours & 0.928 & 0.936 & 0.948 & 0.837 & 0.900 & 0.910 
& 0.984 & 0.943 & 0.964 \\
SS-Ours & 0.986 & 0.987 & 0.986 & 0.967 & 0.989 & 0.983 
& 0.993 & 0.996 & 0.994 \\
%pca-kde 50 & 0.882 & 0.801 & 0.820 & 0.581 & 0.708 & 0.758 & 0.837 & 0.752 & 0.794\\
OC-NN & 0.918 & 0.928 & 0.910 & 0.818 & 0.900 & 0.895 & 0.979 & 0.882 & 0.931 \\
Deep A.Det. & 0.824 & 0.735 & 0.798 & 0.604 & 0.647 & 0.722 &  0.883 & 0.892 & 0.888 \\
DSEBM & 0.676 & 0.630 & 0.633 & 0.462 & 0.662 & 0.613 & 0.460 & 0.571 & 0.516\\
DAGMM & 0.660 & 0.520 & 0.548 & 0.547 & 0.374 & 0.530 & 0.498 & 0.472 & 0.485\\
AD-GAN & 0.382 & 0.601 & 0.618 & 0.462 & 0.430 & 0.499 & 0.548 & 0.519 & 0.534 \\
\hline
\end{tabular}
\caption{AUROC scores for one-class detection task on each class of various datasets using different models. SO and SS implemented with understanding of hierarchical-models is  competitive with
recent deep-learned models across a wide range of tasks and features, with SS being almost always the best. \label{tab:all_auroc}} 
\end{table*}

\subsection{Clustering}
\label{sec:clustering}
The \emph{shell based learning} can also be applied to the problem of unsupervised clustering. The result is illustrated in \fref{fig:clustering}, where we conduct unsupervised clustering on a mixed dataset consisting of images from three classes, namely cat, dog and panda. Images are converted to 2048-dimensional feature vectors using pretrained ResNet50 model followed by unit-normalization operation, after which our shell based learning model searches for three different distinctive shells of the three classes. We plot the distances of each data point to each distinctive shell. Color of the points represents their ground-truth class label. The shells returned by our model are occupied almost solely by feature points of its corresponding label, and the points are separated nicely according to their distances to the distinctive shells. This preliminary experiment reveals the great potential of our \emph{shell based learning} on the clustering problem that is worth further exporing in the next paper.
\begin{figure*}
\centering
\includegraphics[width = 0.3\linewidth]{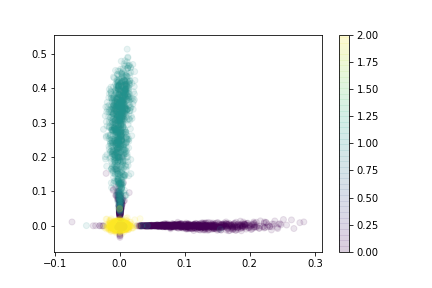}
\includegraphics[width = 0.3\linewidth]{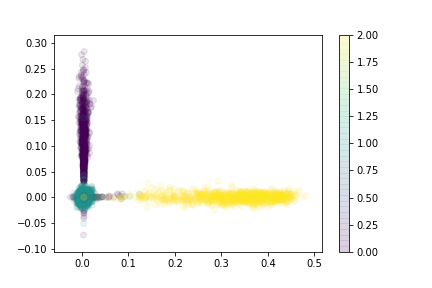}
\includegraphics[width = 0.3\linewidth]{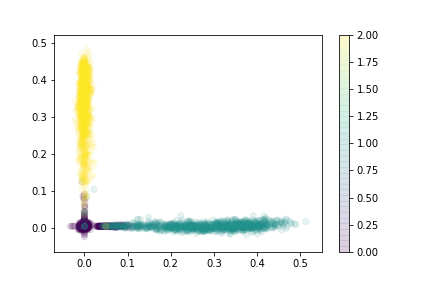}
\caption{Shell based clustering result: x-axis and y-axis value in the plot represents the distance of a data point to the corresponding distinctive shell. From left to right: cat-vs-panda, dog-vs-cat and panda-vs-dog.\label{fig:clustering}}
\end{figure*}

\end{document}